\documentclass{article} 
\pdfoutput=1

\usepackage{nips15submit_e,times}
\usepackage{hyperref}
\usepackage{url}

\usepackage{graphicx}
\usepackage{algorithm}
\usepackage{algorithmic}
\usepackage{hyperref}
\usepackage{breakurl}
\usepackage{amsmath}
\usepackage{amsthm}
\usepackage{amssymb}
\usepackage{url}
\usepackage[caption=false]{subfig}
\usepackage{xr}
\usepackage[english]{babel}

\title{Bounded-Distortion Metric Learning}

\author{
Renjie Liao, Jianping Shi \\
The Chinese University of Hong Kong \\
\texttt{rjliao,jpshi@cse.cuhk.edu.hk} \\
\And
Ziyang Ma \\
University of Chinese Academy of Sciences \\
\texttt{maziyang08@gmail.com} \\
\AND
Jun Zhu \\
Tsinghua University \\
\texttt{dcszj@mail.tsinghua.edu.cn} \\
\And
Jiaya Jia \\
The Chinese University of Hong Kong \\
\texttt{leojia@cse.cuhk.edu.hk} \\
}

\newtheorem{theorem}{Theorem}
\newtheorem{lemma}{Lemma}
\newtheorem{proposition}{Proposition}

\theoremstyle{definition}
\newtheorem{definition}{Definition}
\theoremstyle{remark}
\newtheorem*{remark}{Remark}

\newcommand{\ie}{\textit{i}.\textit{e}.}
\newcommand{\eg}{\textit{e}.\textit{g}.}

\nipsfinalcopy 

\begin{document}

\maketitle

\begin{abstract}
Metric learning aims to embed one metric space into another to benefit tasks like classification and clustering. Although a greatly distorted metric space has a high degree of freedom to fit training data, it is prone to overfitting and numerical inaccuracy. This paper presents {\it bounded-distortion metric learning} (BDML), a new metric learning framework which amounts to finding an optimal Mahalanobis metric space with a bounded-distortion constraint. An efficient solver based on the multiplicative weights update method is proposed. Moreover, we generalize BDML to pseudo-metric learning and devise the semidefinite relaxation and a randomized algorithm to approximately solve it. We further provide theoretical analysis to show that distortion is a key ingredient for stability and generalization ability of our BDML algorithm. Extensive experiments on several benchmark datasets yield promising results.
\end{abstract}

\section{Introduction}\label{sect:introduction}

Distance metric learning is a fundamental problem in machine learning, since many learning algorithms, \eg, k-nearest neighbors (kNN) and k-means, crucially rely on a ``good'' metric. The criteria of good metrics may differ in various learning tasks. For instance, in supervised learning, a common criterion is to learn a metric with a low empirical error~\cite{weinberger2009distance}, while in unsupervised learning, a good criterion is to learn a metric that minimizes the intra-cluster distance and simultaneously maximizes the inter-cluster distance~\cite{xing2003distance}.

In essence, metric learning aims to search a metric embedding to convert the original metric space (\eg, Euclidean) into a new one, which better suits learning tasks with regard to the above criteria. Such an embedding intrinsically induces distortion - a concept in the theory of metric embedding~\cite{bourgain1985lipschitz}, which intuitively measures the effort to reshape the metric space.

Although a large-distorted metric space can have a high degree of freedom to describe data, it may be prone to overfitting. In fact, we will theoretically validate the intuition later by proving that in our case of Mahalanobis metric learning, the generalization bound depends increasingly on the distortion. Moreover, the numerical inaccuracy would also be a problem if the distortion is extremely large. We will show that the distortion of Mahalanobis metric learning is the condition number of the parameter matrix. Inevitably, a large distortion would make the matrix ill-conditioned.

In order to balance the fitness of the learned the metric space to training data and the distortion of the underlying metric embedding, we present {\it bounded-distortion metric learning} (BDML), a generic framework that imposes a bounded-distortion constraint to the learning objective. While it fits various metric learning objectives, we concentrate on learning Mahalanobis metric space, which leads to a semidefinite programming (SDP) formulation.

We approach the SDP via a bisection method, which involves solving a sequence of convex feasibility problems with fast multiplicative weights update~\cite{kale2007efficient}. Moreover, to deal with the pseudo-metric learning, we apply the spectral decomposition to the parameter matrix and perform joint learning of dimension reduction mapping and metric. We relax the resultant non-convex quadratic constrained quadratic programming (QCQP) to a SDP and achieve the approximation by a randomized algorithm. Theoretical analysis is provided to reveal that distortion has a direct impact on the stability and generalization ability of a class of metric learning algorithms. Experimental results on several benchmark datasets manifest the usefulness of our BDML.

\section{Related Work}\label{sect:relatedWork}

Metric learning algorithms can be categorized according to different criteria, such as Mahalanobis~\cite{xing2003distance,shalev2004online,bar2005learning} and non-Mahalanobis~\cite{kedem2012non,norouzi2012hamming,hauberg2012geometric} methods; probabilistic~\cite{goldberger2004neighbourhood,der2012latent} and non-probabilistic~\cite{jin2009regularized,shen2009positive} methods; unsupervised~\cite{roweis2000nonlinear}, supervised~\cite{ying2012distance} and semi-supervised~\cite{xing2003distance} methods; and global~\cite{globerson2005metric} and local~\cite{hastie1996discriminant} methods.

Based on the type of constraints, we can also classify them into pairwise and triplet-wise ones. Pairwise methods~\cite{xing2003distance,davis2007information} often adds constraints to enforce distances between pairs of dissimilar points are larger than a given threshold. Representative methods in the triplet group are the large-margin nearest neighbor~\cite{weinberger2009distance} and its variants~\cite{kedem2012non}. They exploit the local triplet constraints to assure that the distance between any point and its different-class neighbour should be at least one unit margin further than the distance between it and its same-class neighbour. Intuitively, if these triplet constraints are well satisfied, the empirical loss of kNN would be small.

Metric learning is closely related to \emph{metric embedding}, an important topic in theoretical computer science that has played an important role to design approximation algorithms. One line of research focuses on how to embed a finite metric space into normed spaces with a low distortion~\cite{bourgain1985lipschitz,indyk2004low}, \ie, preserving the structure of the original metric space.
Metric learning is also related to manifold learning~\cite{hauberg2012geometric} and kernel learning~\cite{lanckriet2004learning}. Learning a distance metric function amounts to learning a kernel function that measures the similarity between points.


\section{Bounded-Distortion Metric Learning}\label{sect:model}

Before going to details, we introduce necessary notations. $\mathbb{S}^{d} = \{ M | M \in \mathbb{R}^{d \times d}, M^{\top} = M \}$ is the space of all $d \times d$ real symmetric matrices equipped with the Frobenius inner product $A \bullet B = \mathbf{Tr}(A^{\top} B)$. The positive semidefinite (PSD) cone and positive definite (PD) cone are denoted as $\mathbb{S}_{+}^{d} = \{M | M \in \mathbb{S}^{d}, M \succeq 0 \}$ and $\mathbb{S}_{++}^{d} = \{M | M \in \mathbb{S}^{d}, M \succ 0 \}$ respectively. The convex set $\mathbb{P}_{R}^{d} = \{M | M \in \mathbb{S}_{++}^{d}, \mathbf{Tr}(M) \le R\}$ is also used. The trace bound $R$ in $\mathbb{P}_{R}^{d}$ is a parameter to ensure a bounded domain for $M$. $\mathbb{R}_{+}^{d}$ denotes the $d$-dimensional nonnegative orthant.

\subsection{Distortion of Metric Embedding}\label{sect:embedding}

A definition of metric space is the following.

\begin{definition}\label{metricSpace}
A pair $(\mathcal{X}, d_{\mathcal{X}})$ is called a metric space, where $\mathcal{X}$ is a set of points and $d_{\mathcal{X}} :\mathcal{X} \times \mathcal{X} \to [0, \infty)$ is a distance function satisfying the following conditions for all $x_i, x_j, x_k \in \mathcal{X}$:
\begin{itemize}
\item $d_{\mathcal{X}}(x_i, x_j) = 0$, iff $x_i = x_j$,
\item $d_{\mathcal{X}}(x_i, x_j) = d_{\mathcal{X}}(x_j, x_i)$,
\item $d_{\mathcal{X}}(x_i, x_j) + d_{\mathcal{X}}(x_j, x_k) \ge d_{\mathcal{X}}(x_i, x_k)$.
\end{itemize}
\end{definition}

\begin{figure}[t]
  \centering
    \begin{tabular}{@{\hspace{1mm}}c@{\hspace{.5cm}}c@{\hspace{.5cm}}c@{\hspace{1mm}}}
    \subfloat[]{\includegraphics[height=2.6cm]{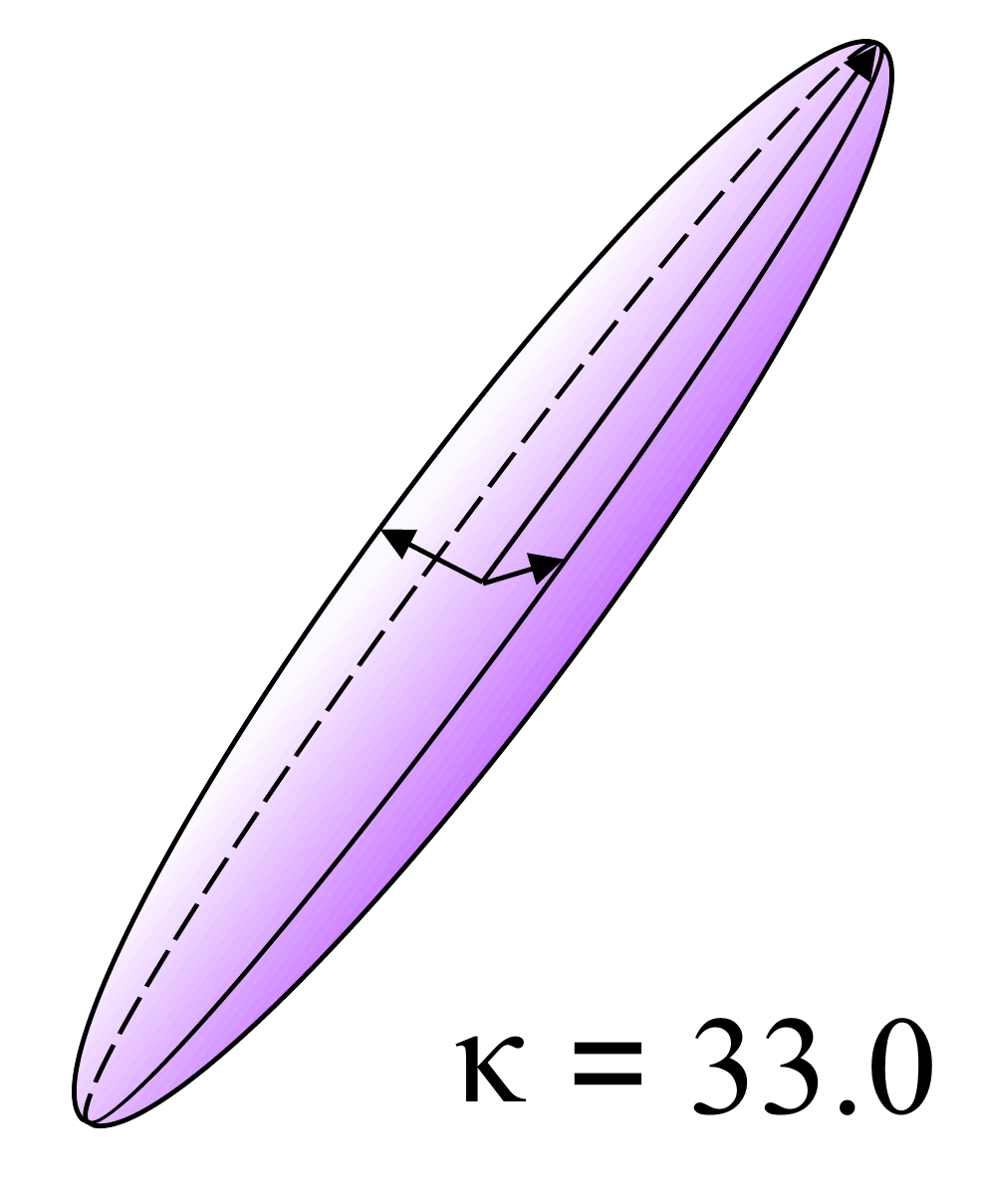}} &
    \subfloat[]{\includegraphics[height=2.35cm]{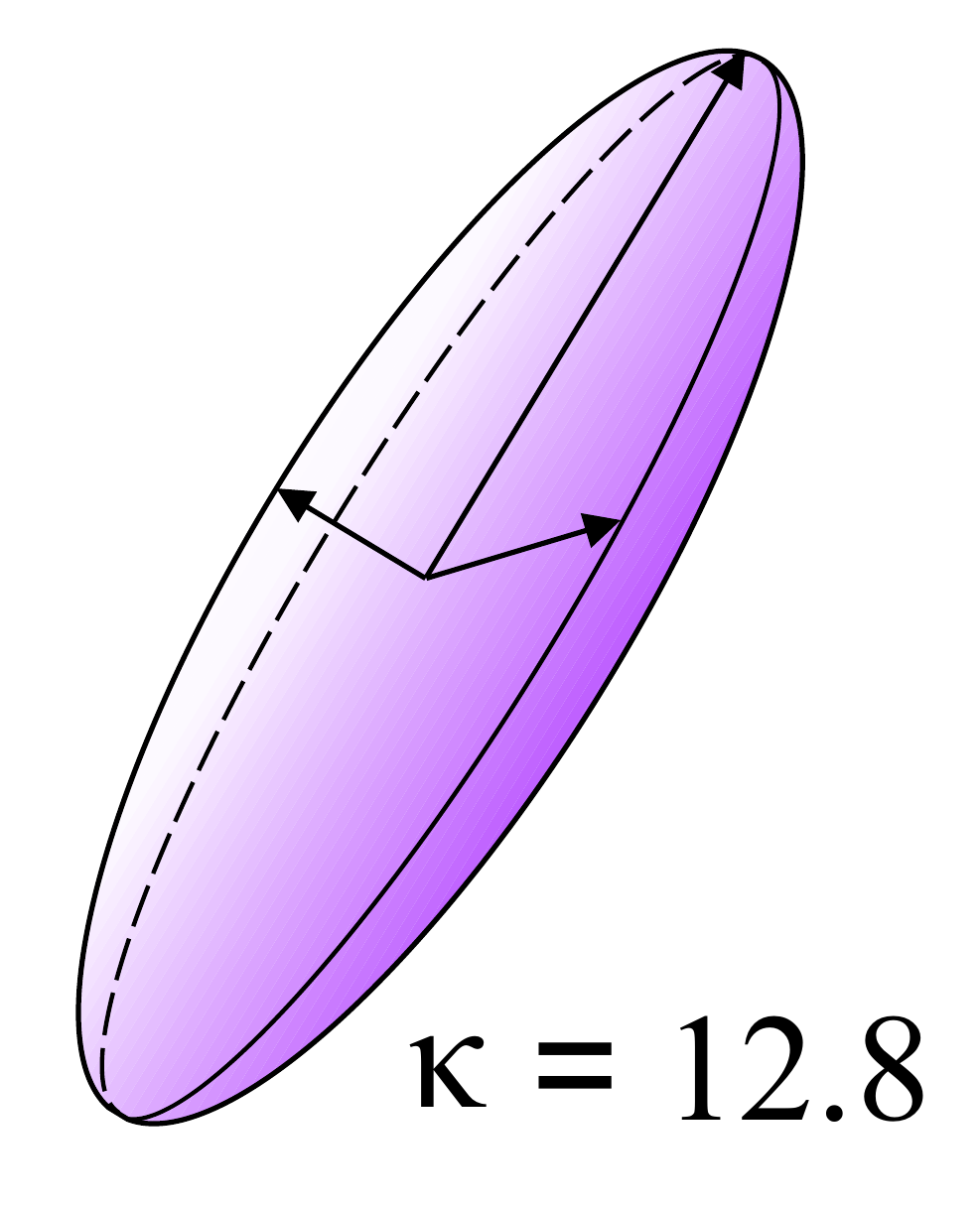}} &
    \subfloat[]{\includegraphics[height=2.1cm]{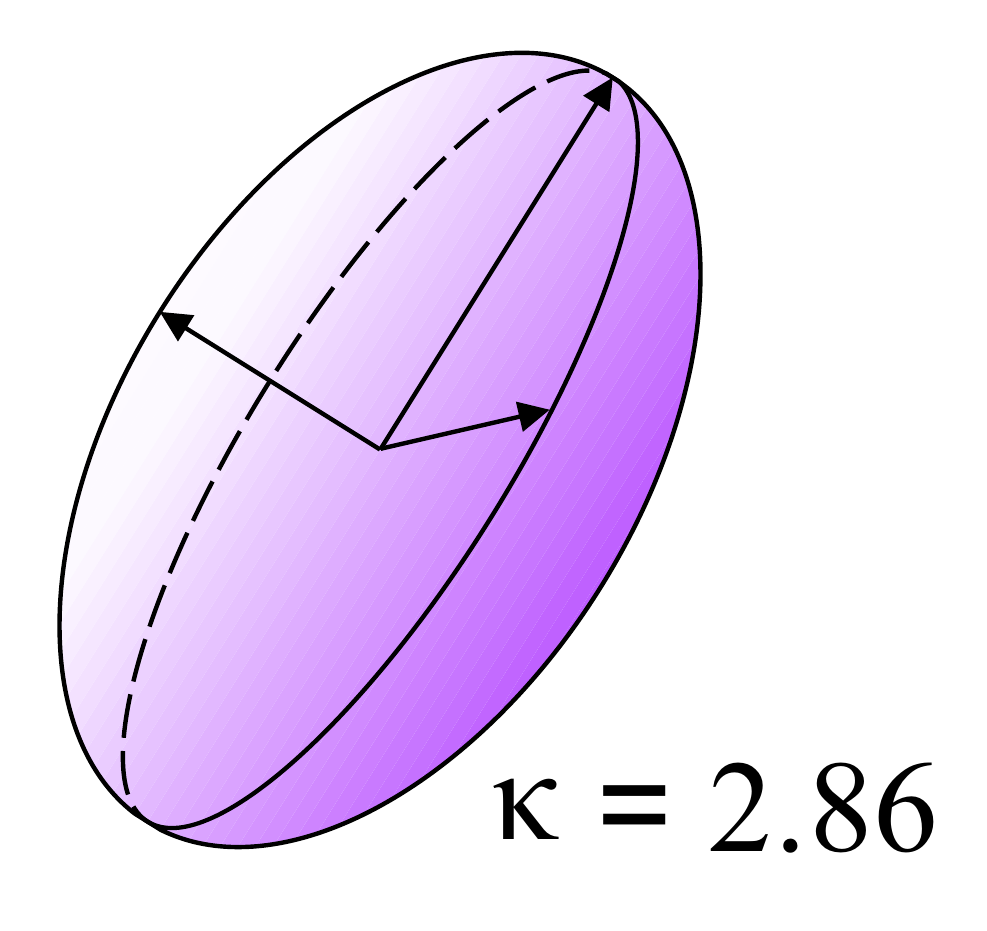}}
    \end{tabular}
  \caption{Ellipsoids with various condition numbers, where (a) and (c) have the same logDet value $\approx$ 3.5; (b) and (c) have the same F-norm value $\approx$ 7.2.}
  \label{fig:Visualization}
\end{figure}
A Mahalanobis metric space is a metric space equipped with a Mahalanobis distance function, which often takes the form of $d_{\mathcal{X}}(x_i,x_j) = \sqrt{{(x_i - x_j)^{\top}}M(x_i - x_j)}$, parameterized by a PD matrix $M$, \ie, $M \in \mathbb{S}_{++}^{d}$. In the metric learning literatures, a PSD $M$ is usually adopted, thus the induced distance function is a pseudo-metric in the strict sense. We now focus on the PD case and defer the PSD one to Sect.~\ref{sect:pseudo_metric}. Obviously, the Euclidean space is a special Mahalanobis metric space where $M$ is an identity matrix. Note that we deal with the squared Mahalanobis distance since it does not affect learning methods (\eg, kNN) that are based on relative distances.

We can embed one metric space into another with a certain degree of distortion. The formal definition of metric embedding and its distortion are as follows~\cite{bourgain1985lipschitz,indyk2004low},

\begin{definition}\label{embedding}
Let $(\mathcal{X}, d_{\mathcal{X}})$ and $(\mathcal{Y}, d_{\mathcal{Y}})$ be two metric spaces. A mapping $f :\mathcal{X} \to \mathcal{Y}$ is said to be a $c$-embedding if there exists $r > 0$ such that for all $x,y \in \mathcal{X}$,
\begin{equation}\label{eq:embedding}
r \cdot d_{\mathcal{X}}(x,y) \le d_{\mathcal{Y}}(f(x),f(y)) \le cr \cdot d_{\mathcal{X}}(x,y).
\end{equation}
The distortion of $f$ is defined as the infimum of all $c$ such that $f$ is a $c$-embedding.
\end{definition}

Distortion is a measure of the distance between two metric spaces and plays an important role in the theory of metric embeddings. Later we will show that distortion is essential to stabilize a class of metric learning algorithms.

In Mahalanobis metric learning, given an Euclidean metric space $(\mathcal{X}, d_{I})$, we learn a metric embedding $f_{I \to M}$, which returns us a desired Mahalanobis metric space $(\mathcal{X}, d_{M})$. We have the following proposition to specify the distortion of this metric embedding.

\begin{proposition}\label{pro:distortion}
The distortion of the metric embedding $f_{I \to M}$ is the condition number $\kappa(M)$.
\end{proposition}

Due to the page limit, we focus on presenting important results in this paper and defer all proofs to the appendix in the supplementary file.

\subsection{Geometric Meaning of Distortion}\label{subsect:GeometryMeaning}

Distortion can be intuitively regarded as a complexity measure of metric embedding. From a geometric perspective, it possesses an intrinsically different meaning compared to other complexity measures in previous work, including the log determinant (logDet)~\cite{davis2007information} and Frobenius norm (F-norm)~\cite{jin2009regularized}. Specifically, we focus on analyzing the metric embedding $f_{I \to M}$ and consider an origin-centered ellipsoid $\mathcal{E} = \{x \in \mathbb{R}^d | x^{\top}Mx \le 1\}$ for simplicity.

Let $\{\lambda_i\}_{i=1}^d$ be the set of eigenvalues of $M$. It is well-known that the logarithmic volume of $\mathcal{E}$ is $\log(V(\mathcal{E})) = \log(\gamma) - \frac{1}{2}\text{logDet}(M)$, where $\gamma$ is the volume of the unit sphere in $\mathbb{R}^d$. The squared F-norm of $M$ is defined as $ \left\| M \right\|_F^2 = \sum\nolimits_{i = 1}^d {1/r_i^4}$, where $r_i = 1/\sqrt{\lambda_i}$ is the length of the $i$-th semi-axis. The condition number is $\kappa(M) = r_{max}^2 / r_{min}^2$. In other words, logDet measures the volume variation; F-norm indicates the change of overall lengths of semi-axes; while the condition number describes the length ratio variation between the longest and the shortest semi-axis. As illustrated in Fig.~\ref{fig:Visualization}, it is possible that a metric with a small logDet or F-norm value is ill-conditioned, \ie, the ellipsoid is extremely elongated. Therefore, rather than focusing on the absolute variation, distortion measures the relative one, thus enabling higher freedom and directly controling the well-conditioning property.

\subsection{Pair and Triplet Constrained BDML}\label{subsect:BDML}

Given a training dataset $D = \{(x_i, y_i)\}_{i=1}^{N}$, where $x_i \in \mathbb{R}^d$ is a data point and $y_i$ is the corresponding class label, our task is to obtain a Mahalnobis distance metric space $(\mathcal{X}, d_M)$, where $d_M$ is the distance function defined as, $d_M(x_i, x_j) = (x_i - x_j)^{\top}M(x_i - x_j) = M \bullet {X_{ij}}$, and $M \in \mathbb{S}_{++}^d$. Here we define $X_{ij} = (x_i - x_j)(x_i - x_j)^{\top}$ for notation simplicity.

We now present two notable formulations of our bounded-distortion metric learning (BDML), which correspond to two types of constraints in the literature of metric learning, \ie, pairwise ones and triplet-wise ones. Specifically, for any data point $x_i$, we consider its $k$-nearest neighbours set $\Omega_i$. Following~\cite{weinberger2009distance}, two types of neighbor points are distinguished. They are \emph{target neighbors} that share the same class label with $x_i$, and \emph{imposter neighbors} that have different class labels with $x_i$.

Let $\mathcal{S} = \{ (i, j) | x_j \in \Omega_i, y_j = y_i \}$ and $\mathcal{I} = \{ (i, j) | x_j \in \Omega_i, y_j \neq y_i \}$ be the sets of all index pairs of target neighbors and impostor neighbors respectively. $\mathcal{T} = \{(i, j, k)|(i, j) \in \mathcal{S}, (i, k) \in \mathcal{I}\}$ denotes the set of all such index triplets.

We define the pair-constrained bounded-distortion metric learning (\textit{p}-BDML) as
\begin{align}\label{eq:BDML_pair}
{\mathop {\min }\limits_{M \in \mathbb{P}_{R}^{d}} } \qquad & { \frac{1}{n} \sum\nolimits_{(i,j) \in \mathcal{S}} {M \bullet {X_{ij}}}} \nonumber \\
{s.t.} \qquad & {M \bullet {X_{ij}} \ge \mu}, \qquad {\forall (i,j) \in \mathcal{I}}, \nonumber \\
& \kappa(M) \le K,
\end{align}
where $n = |\mathcal{S}|$ and $K$ is a parameter to control the upper bound of distortion. Pair-wise constraints are designed to pull two imposter neighbors farther than a margin or push two target neighbors closer than a margin in literatures. Here we only consider the former purpose and minimize the average distance of target neighbors as in \cite{xing2003distance}.

The triplet-constrained bounded-distortion metric learning (\textit{t}-BDML) is formulated as,
\begin{align}\label{eq:BDML_triplet}
{\mathop {\min }\limits_{M \in \mathbb{P}_{R}^{d}} } \qquad & { \frac{1}{n} \sum\nolimits_{(i,j) \in \mathcal{S}} {M \bullet {X_{ij}}}} \nonumber \\
{s.t.} \qquad & {M \bullet {X_{ik}} - M \bullet {X_{ij}} \ge \mu}, \quad {\forall (i,j,k) \in \mathcal{T}}, \nonumber \\
& \kappa(M) \le K.
\end{align}
The above inequality constraint of triplet nearest neighbors ensure that any given point has its impostor neighbors at least one unit margin farther than its target neighbors. Note that the unit margin in~\cite{weinberger2009distance} can be set as a arbitrarily positive constant, since it only affects the scale of $M$. In our case, since we consider a bounded domain of $M$ (\ie, $M \in \mathbb{P}_{R}^{d}$), the margin $\mu$ is treated as a positive parameter.

Note that the bounded-distortion constraint $\kappa(M) \le K$ implies that $M$ should be PD since otherwise $\kappa(M)$ is unbounded. Albeit nonconvex, the condition number function is quasi-convex. It means all its sublevel sets are convex. This property enables us to transform \textit{p}-BDML and \textit{t}-BDML  to the standard formulation of SDP.



\section{A Bisection Algorithm with Multiplicative Weights Update}\label{sect:solver}

We now present a bisection algorithm for approaching our \textit{p}-BDML and \textit{t}-BDML, which essentially solves a sequence of convex feasibility problems. For each feasibility problem, we resort to the multiplicative weights update (MWU) method~\cite{kale2007efficient,arora2012multiplicative}, which is a meta algorithm and has many variants in different disciplines. The reason of choosing MWU is that it generates an approximate solution with guaranteed constraint violation -- it is important for the analysis in Sec.~\ref{sect:general_bounds} to hold.

\begin{algorithm}[t]
\caption{: A Bisection Method}\label{alg:bisection}
\begin{algorithmic}[1]
\STATE Given the interval of $g^{*}$ as $[L, U]$, tolerance $\epsilon > 0$,
\STATE \textbf{Repeat}
\STATE \qquad ${\bar{g}} = (L + U)/2$.
\STATE \qquad Solve the convex feasibility problem~(\ref{eq:feasibilityProb}).
\STATE \qquad \textbf{If} problem~(\ref{eq:feasibilityProb}) is feasible\textbf{:} \qquad $U = {\bar{g}}$.
\STATE \qquad \textbf{Else}\textbf{:} \qquad \qquad \qquad \qquad \qquad $L = {\bar{g}}$.
\STATE \textbf{Until} $U - L \leq \epsilon$
\STATE Return the final objective value as ${\bar{g}}$.
\end{algorithmic}
\end{algorithm}

\subsection{Sequential Convex Feasibility Problems}\label{sect:seq_convex}

In what follows, we only describe the convex feasibility problem for \textit{p}-BDML, since the formulation for \textit{t}-BDML only differs in constraints. We denote the objective function as $g(M) = G \bullet M$ where $G = \frac{1}{n} \sum\nolimits_{(i,j) \in \mathcal{S}} X_{ij}$. Its optimal value $g^{*}$ is assumed to lie in the initial interval $[L, U]$, where $L$ and $U$ are set as $0$ and $g(I)$ respectively. $I$ is the identity matrix. Our bisection algorithm estimates $g$ and narrows down the interval by half in each iteration. The procedure of the bisection algorithm is outlined in Alg.~\ref{alg:bisection}.

Specifically, if $g$ is not larger than ${\hat g}$ in one iteration, we solve a convex feasibility problem as
\begin{align}\label{eq:feasibilityProb}
find \qquad & M \in \mathbb{P}_{R}^{d}, \quad \alpha > 0 \nonumber \\
{s.t.} \qquad & {G \bullet M} \le {\bar{g}}, \nonumber \\
& {M \bullet X_{ij} \ge \mu}, \quad {\forall (i,j) \in \mathcal{I}}, \nonumber \\
& {\alpha I \preceq M \preceq \alpha KI}.
\end{align}
Here we introduce a positive auxiliary variable $\alpha$ and transform the bounded-distortion constraint into two generalized inequality constraints. The resultant convex feasibility problem can be approximately solved by the efficient MWU method, to be elaborated on in the next section.

\subsection{Multiplicative Weights Update Method}

Before applying MWU method, we reformulate the feasibility problem~(\ref{eq:feasibilityProb}) to a general form via introducing slack variables ${M_1} = M - \alpha I$ and ${M_2} = \alpha KI - M$. Then we construct a sparse symmetric matrix $Y \in \mathbb{P}_{R}^{3d+1}$ of which the block diagonal entries are $M$, $M_1$, $M_2$ and $\alpha$. All constraints except $\mathbb{P}_{R}^{3d+1}$ in~(\ref{eq:feasibilityProb}) are rewritten as ${J_i} \bullet Y \ge {h_i}$. $\mathbb{P}_{R}^{3d+1}$ can be deemed as an easy constraint, contrary to each hard constraint ${J_i} \bullet Y \ge {h_i}$. With this change, we obtain the equivalent formulation of (\ref{eq:feasibilityProb}) as
\begin{align}\label{eq:feasibilityProbSDP}
find    \qquad  & Y \in \mathbb{P}_{R}^{3d+1} \nonumber \\
{s.t.}  \qquad  & {{J_i} \bullet Y \ge {h_i}}, \quad \forall i = 1,...,m.
\end{align}
The number of constraints is $m = |\mathcal{I}|+4d^2+2$. We also introduce a closely related feasibility problem as
\begin{align}\label{eq:feasibilityProbORACLE}
find    \qquad  & Y \in \mathbb{P}_{R}^{3d+1} \nonumber \\
{s.t.}  \qquad  & \sum\nolimits_{i = 1}^m {p_i\left( {{J_i} \bullet Y - {h_i}} \right)}  \ge 0.
\end{align}
Here $\mathbf{p} = [p_{1}, ..., p_{m}]^{\top}$ is a probability vector, \ie, $\forall i, p_i \ge 0$ and $\sum\nolimits_{i = 1}^m p_i = 1$. Note that this problem only contains 2 constraints, easy to solve. The relationship between problem~(\ref{eq:feasibilityProbSDP}) and~(\ref{eq:feasibilityProbORACLE}) is summarized by the following lemma.
\begin{lemma}\label{lem:feasibilityEquivalence}
If problem~(\ref{eq:feasibilityProbSDP}) has a feasible solution $Y^{*}$, given any probability $\mathbf{p}$, then $Y^{*}$ is feasible for problem \eqref{eq:feasibilityProbORACLE}. Equivalently, if there exists a probability $\mathbf{p}$ such that problem \eqref{eq:feasibilityProbORACLE} is infeasible, then problem \eqref{eq:feasibilityProbSDP} is infeasible.
\end{lemma}

With the lemma, we now describe the MWU method for approaching problem \eqref{eq:feasibilityProbSDP}. Basically, MWU maintains a weight vector $\mathbf{w} \in \mathbb{R}_{+}^{m}$, where each entry $w_i$ represents the importance of the $i$-th constraint. It iteratively solves problem \eqref{eq:feasibilityProbORACLE} and updates weights $\mathbf{w}$ according to the constraint satisfaction ${J_i} \bullet Y - {h_i}$. Intuitively, if one constraint is more satisfied, the corresponding importance should be less and we should decrease its weight.

In $t$-th round of the algorithm, we get a probability vector $\mathbf{p}^{(t)}$ by normalizing the nonnegative weights $\mathbf{w}^{(t)}$. Then we solve the $2$-constraint feasibility problem \eqref{eq:feasibilityProbORACLE} by maximizing $\sum\nolimits_{i = 1}^m {p_i\left( {{J_i} \bullet Y^{(t)} - {h_i}} \right)}$ over $\mathbb{P}_{R}^{3d+1}$. If the maximum value is greater than $0$, we take the corresponding $Y^{(t)}$ as a feasible solution. Otherwise \eqref{eq:feasibilityProbORACLE} is infeasible. We call the solver of problem \eqref{eq:feasibilityProbORACLE} as an \textsc{Oracle}.

Implementing \textsc{Oracle} needs to compute the largest eigenvector of the matrix $C = \sum\nolimits_{i = 1}^m {p_i[{J_i} - (h_i/R)I]}$, which can be efficiently handled by Lanczos algorithm. Here $R$ is the trace bound parameter in $\mathbb{P}_{R}^{3d+1}$.

Assuming \textsc{Oracle} obtains a feasible solution $Y^{(t)}$, we denote the normalized satisfaction of $i$-th constraint of problem \eqref{eq:feasibilityProbSDP} as $\eta_i^{(t)} = \frac{1}{\rho}[{J_i} \bullet Y^{(t)} - {h_i}]$. Here $\rho$ is called the \emph{width} parameter, satisfying that $\forall i$, $\left| {{J_i} \bullet {Y^{(t)}} - {h_i}} \right| \le \rho$. We update the weights as $w_i^{(t+1)} = w_i^{(t)}(1 - \varepsilon \eta_i^{(t)})$, where $\varepsilon$ is a parameter smaller than $1/2$. Hence, $w_i^{(t+1)}$ is smaller than $w_i^{(t)}$ if $\eta_i^{(t)} > 0$, and its value increases otherwise.

The algorithm is depicted in Alg.~\ref{alg:MWU}. After $T$ rounds, the averaged solution $\bar Y = (\sum\nolimits_{t = 1}^T {{Y^{(t)}}} )/T$ is returned. We have the Theorem~\ref{the:correctnessMWU} following~\cite{arora2012multiplicative} to guarantee that either $\bar Y$ achieves a predefined accuracy or we claim that the original problem \eqref{eq:feasibilityProbSDP} is infeasible.

\begin{theorem}\label{the:correctnessMWU}
Let $\delta\!>\!0$ be a given additive error\footnote{Additive error up to $\delta$ means that, any constraint is violated at most $\delta$, \ie, $\forall i, {J_i} \bullet Y - {h_i} \ge -\delta$.}. Alg.~\ref{alg:MWU} either solves~problem~(\ref{eq:feasibilityProbSDP}) up to $\delta$, or correctly concludes that it is infeasible, making $O(\frac{{{\rho ^2}\ln (m)}}{{{\delta ^2}}})$ calls to the \textsc{Oracle}.
\end{theorem}

\begin{algorithm}[t]
\caption{: Multiplicative Weights Update Method}\label{alg:MWU}
\begin{algorithmic}[1]
\STATE \textbf{Initialization:} Fix a $\varepsilon \le 1/2$, for each constraint, associate the weight $w_i^{(1)} = 1$.
\STATE \textbf{For} $t = 1, 2, ..., T$:
\STATE \qquad Normalize $\mathbf{w}^{(t)}$ to get the probability vector $\mathbf{p}^{(t)}$.
\STATE \qquad Call the \textsc{Oracle} with $\mathbf{p}^{(t)}$.
\STATE \qquad \textbf{If} \textsc{Oracle} succeeds to find a solution $Y^{(t)}$
\STATE \qquad \qquad $\eta_i^{(t)} = \frac{1}{\rho}[{J_i} \bullet Y^{(t)} - {h_i}]$.
\STATE \qquad \qquad $w_i^{(t+1)} = w_i^{(t)}(1 - \varepsilon \eta_i^{(t)})$.
\STATE \qquad \textbf{Else}
\STATE \qquad \qquad Return that the problem is infeasible.
\STATE \qquad \textbf{End}
\STATE \textbf{End}
\STATE Return $\bar Y = (\sum\nolimits_{t = 1}^T {{Y^{(t)}}} )/T$ as a final solution.
\end{algorithmic}
\end{algorithm}


\section{Pseudo-Metric \& Dimension Reduction}\label{sect:pseudo_metric}

In this section, we deal with the case of pseudo-metric, \ie, $M$ is PSD. In the literature of Mahalanobis metric learning, a PSD $M$ is beneficial due to the existence of decomposition $M = C^{\top}C$, where $C \in \mathbb{R}^{q \times d}$. The distance function could be rewritten as $d(x, y) = ||Cx - Cy||^2$. It thus removes the PSD constraint and allows flexible dimension reduction by choosing $q < d$.

In our setting, a PSD $M$ could be problematic if it has an unbounded condition number. According to the spectral theorem, decomposition $M = Q^{\top} \Lambda Q$ is applicable, where $\Lambda \in \mathbb{R}^{q \times q}$ is a diagonal matrix with eigenvalues of $M$, $Q \in \mathbb{R}^{q \times d}$ has orthonormal rows and $q$ is the rank of $M$. We can also choose different $q$ to form different low rank approximation of $M$. Hence $Q$ and $\Lambda$ act as a dimension-reduction mapping and a dimension-wise scaling operation respectively.

By replacing $M$ with the decomposition and adding orthogonal constraints for $Q$ in original BDML, we can conduct the pseudo-metric learning via alternatively optimizing $Q$ and $\Lambda$. Specifically, when $Q$ is fixed, optimizing the diagonal $\Lambda$ is just a simple case of original BDML. Nevertheless, optimizing $Q$ with a known $\Lambda$ is not straightforward. Especially, in the case of \textit{p}-BDML, when $\Lambda$ is fixed such that $\Lambda \in \mathcal{P}$ and $\kappa(\Lambda) \le K$, the problem \eqref{eq:BDML_pair} can be reformulated as below,
\begin{align}\label{eq:BDML_pseudo_metric}
{\mathop {\min }\limits_{Q \in \mathbb{R}^{q \times d}} } \qquad & { \frac{1}{n} \sum\nolimits_{(i,j) \in \mathcal{S}} {{X_{ij}} \bullet (Q^{\top} \Lambda Q)}} \nonumber \\
{s.t.} \qquad & {{{X_{ij}} \bullet (Q^{\top} \Lambda Q)} \ge \mu}, \qquad {\forall (i,j) \in \mathcal{I}}, \nonumber \\
& QQ^{\top} = I.
\end{align}
Note this learning problem is nontrivial due to the fact that optimizing $Q$ is a quadratic constrained quadratic programming (QCQP). To overcome the difficulty, we have the following proposition,

\begin{proposition}\label{pro:mapping}
Problem \eqref{eq:BDML_pseudo_metric} can be relaxed to a SDP as following,
\begin{align}\label{eq:BDML_pseudometric_relax}
{\mathop {\min }\limits_{\tilde{Q} \in \mathbb{S}_{+}^{qd}} } \qquad & {\frac{1}{n} \sum\nolimits_{(i,j) \in \mathcal{S}} {\tilde{X}_{ij} \bullet \tilde{Q}}} \nonumber \\
{s.t.} \qquad & {\tilde{X}_{ij} \bullet \tilde{Q}} \ge \mu, \qquad & {\forall (i,j) \in \mathcal{I}}, \nonumber \\
& {{A}_{uv} \bullet \tilde{Q}} = b_{uv}, & {\forall (u, v) \in \mathcal{C}},
\end{align}
where $\tilde{X}_{ij} = X_{ij} \otimes \Lambda$ and $\otimes$ stands for Kronecker product. ${A}_{uv}$ is a block diagonal matrix which contains $d$ identical blocks $B_{uv} \in \mathbb{R}^{q \times q}$. $(u, v)$ and $(v, u)$-th entries of $B_{uv}$ are 1 while others are 0. $b_{uv} = 2$ if $u = v$, otherwise $b_{uv} = 0$. $\mathcal{C} = \{(u, v) \in [q] \times [q] | u \le v \}$ and $[q] = \{1, \dots, q\}$.
\end{proposition}
This proposition shows that when $\Lambda$ is fixed we can learn $Q$ by solving the above SDP relaxation. Moreover, since the equality constraints in \eqref{eq:BDML_pseudometric_relax} imply $\mathbf{Tr}(\tilde{Q}) = q$, we thus can exploit the MWU method again.

Denoting the optimal objective values of problem \eqref{eq:BDML_pseudo_metric} and \eqref{eq:BDML_pseudometric_relax} as $\varrho_{qp}$ and $\varrho_{sdp}$ respectively, it is obvious that $\varrho_{sdp} \le \varrho_{qp}$. Hence we aim at up-bounding $\varrho_{qp}$. Specifically, once we obtained the optimal solution ${\tilde{Q}}^{*}$ of problem \eqref{eq:BDML_pseudometric_relax}, we construct an approximate solution $\xi$ problem \eqref{eq:BDML_pseudo_metric} based on a Gaussian randomization procedure shown in Alg.~\ref{alg:gaussian_random}.
We prove the following theorem to assure that in the worst case Alg.~\ref{alg:gaussian_random} would possibly generate an approximate
solution with approximation ratio $\omega$.
\begin{theorem}\label{thm:tailbound}
If the optimal solution of problem \eqref{eq:BDML_pseudometric_relax} is $\tilde{Q}^{*}$ and $\xi \in \mathbb{R}^{qd}$ is a random vector generated from the real-valued normal distribution $\mathcal{N}(0, {\tilde{Q}}^{*})$, then for any $\gamma > 0$, $\epsilon \ge 0$ and $\omega \ge 1$, we have,
{\footnotesize
\begin{align}
& Prob \left(\nu \ge \gamma \mu ~~~ \& ~~~ \zeta \le \epsilon ~~~ \& ~~~ \xi^{\top} \tilde{G} \xi \le \omega \tilde{G} \bullet \tilde{Q}^{*} \right) \ge 1 - |\mathcal{I}| \max \left( \sqrt{\gamma}, \frac{2(r-1)\gamma}{\pi - 2} \right) \nonumber \\
& - r \exp{\left(-\frac{1}{2}\left(\omega-\sqrt{2\omega - 1}\right)\right)} - \frac{rq(q+1)}{2} \left[ \exp{\left(-\frac{(\tau - 1)^2}{4}\right)} + \exp{\left(-\frac{\epsilon^2}{8rdq^2}\right)} \right],  \nonumber
\end{align}}
where $r = rank(\tilde{Q}^{*})$ and $\tau = \sqrt{\frac{\epsilon}{q} \left(\frac{2}{rd}\right)^{1/2} + 1}$. $\nu$, $\zeta$ and $\tilde{G}$ are defined respectively as $\nu = \min_{(i,j) \in \mathcal{I}} \xi^{\top} \tilde{X}_{ij} \xi$, $\zeta = \max_{(u,v) \in \mathcal{C}} |\xi^{\top} {A}_{uv} \xi - b_{uv}|$ and $\tilde{G} = \frac{1}{n} \sum\nolimits_{(i,j) \in \mathcal{S}} \tilde{X}_{ij}$.
\end{theorem}

\begin{algorithm}[t]
\caption{: Gaussian Randomization Procedure}\label{alg:gaussian_random}
\begin{algorithmic}[1]
\STATE \textbf{Initialization:} Given the optimal solution ${\tilde{Q}}^{*}$, iteration number $T^{\prime}$, ratio $\gamma$ and tolerance $\epsilon$.
\STATE \textbf{For} $t = 1, 2, ..., T^{\prime}$:
\STATE \qquad Sample ${\xi_t} \sim \mathcal{N}(0, {\tilde{Q}}^{*})$.
\STATE \textbf{End}
\STATE $\xi = {\arg \min}_{\xi_t} ~~ {\frac{1}{n} \sum\nolimits_{(i,j) \in \mathcal{S}} { \xi_t^{\top} \tilde{X}_{ij}  \xi_t }}$.
\STATE Reshape $\xi$ from $\mathbb{R}^{qd \times 1}$ to $\mathbb{R}^{q \times d}$.
\STATE Return $\xi$ as the approximate solution of problem \eqref{eq:BDML_pseudo_metric}.
\end{algorithmic}
\end{algorithm}

\begin{remark}\label{thm_approx_acc}
This theorem indicates that with well chosen $\gamma$ and $\epsilon$, Alg.~\ref{alg:gaussian_random} can generate an approximate solution for \eqref{eq:BDML_pseudo_metric} with guaranteed approximation ratio even in the worst case. For example, we can consider a real case with $q = 10$, $d = 100$ and the number of constraints $|\mathcal{I}| = 100$. By choosing $\gamma = \pi/16|\mathcal{I}|^2$, $\epsilon = 40q\sqrt{rd}$ and with appropriate rank reduction on $\tilde{Q}^{*}$ as~\cite{luo2007approximation}, it can be shown that after running Alg.~\ref{alg:gaussian_random} for $100$ iteration, we have very high probability\footnote{The probability is at least $0.999828.$} such that $\frac{1}{\omega} \varrho_{qp} \le \varrho_{sdp} \le \varrho_{qp}$, where $\omega = 10$. However, the price we pay is that the orthogonal constraints are loosely satisfied. In practice, we found that the resultant approximate solution works well, which indicates that removing orthogonal constraints and transforming $\Lambda$ to a full rank matrix may also be an alternative modeling choice.

As for the pseudo-metric learning of \textit{t}-BDML, we can still use the above algorithm to obtain an approximate solution. However, Theorem~\ref{thm:tailbound} does not stand in this case since not all $\tilde{X}_{ij}$ of \textit{t}-BDML are PSD.
\end{remark}


\section{Generalization Bound for BDML}\label{sect:general_bounds}

To theoretically investigate whether the distortion has an impact on the generalization ability, we derive the generalization bound of our BDML following the stability analysis of learning algorithms~\cite{bousquet2002stability,shalev2010learnability}.

Before diving into the details, we first introduce one assumption that we only consider the case where the metric matrix is full rank. Then we clarify some preliminary notations. Each training sample $z$ inside the training set $D$ is drawn i.i.d. from some unknown distribution $\mathcal{D}$. And the range of $z$ is denoted as $\mathcal{Z}$. $D^{i}$ is a perturbed set of $D$ obtained via replacing $i$-th sample with a new sample drawn from $\mathcal{D}$, \ie, $D^{i} = \{D \backslash z_i \cup z_i^{\prime}\}$, where $z_i^{\prime} \sim \mathcal{D}$. We make a mild assumption that all data points are contained in a $\Gamma$-ball, \ie, ${\left\| x \right\|_2} \le \Gamma$. We denote the learning algorithm as $\mathcal{A}$, the \emph{true risk} or \emph{generalization error} as $\mathcal{R}(\mathcal{A}, D)$ and the \emph{empirical risk} as $\mathcal{R}_{emp} (\mathcal{A}, D) = \frac{1}{n}\sum\nolimits_{k} {\ell(\mathcal{A},{z_{k}})}$, where in our case the loss function $\ell = M \bullet {X_{ij}}$ and $n = |\mathcal{S}|$. Based on~\cite{bousquet2002stability}, we define the Uniform-Replace-One stability as,
\begin{definition}\label{uniformRO}
An algorithm $\mathcal{A}$ has Uniform-Replace-One stability $\beta$ with respect to the loss function $\ell$ if $\forall D \in \mathcal{Z}^{n}$, $\forall i \in \{1,...,n\}$,
\begin{align}
||\ell(\mathcal{A}_D,\cdot) - \ell(\mathcal{A}_{D^{i}},\cdot)||_{\infty} \le \beta.
\end{align}
\end{definition}
Here $\mathcal{A}_D$ means the learning algorithm $\mathcal{A}$ is trained on the dataset $D$. Note that our definition is stronger than the one proposed in~\cite{shalev2010learnability}, thus being more restrictive. We have the following lemma, which specifies the uniform-RO stability of our BDML.
\begin{lemma}\label{lem:stability_BDML}
The Uniform-Replace-One stability of our BDML algorithm with respect to the given loss function $\ell$ is $\beta = \frac{4(K+1)R\Gamma^2}{d}$.
\end{lemma}
This lemma indicates that the Uniform-Replace-One stability of $\mathcal{A}$ is positively correlated with the bound of distortion $K$. It means a low distortion of metric embedding would lead to a stable algorithm. Although $\beta$ does not depend decreasingly on the number of samples $n$, which may not be seen as \emph{stable} in some sense~\cite{bousquet2002stability}, it is clear that the stability can be controlled by the distortion. With this stability result, we further prove the following generalization bound, which theoretically explains the relationship between distortion and generalization error.
\begin{theorem}\label{thm:polybound}
For any metric learning algorithm $\mathcal{A}$ with Uniform-Replace-One stability $\beta$ with respect to the given loss function $\ell$, we have with probability at least $1 - \delta$,
\begin{align}
\mathcal{R}(\mathcal{A}, D) \le \mathcal{R}_{emp} (\mathcal{A}, D) + 2\Gamma\sqrt{\frac{KR}{d\delta} \left(\frac{2KR\Gamma^2}{nd} + 3\beta\right)}. \nonumber
\end{align}
Specifically, for our BDML algorithm,
\begin{align}
\mathcal{R}(\mathcal{A}, D) \le \mathcal{R}_{emp} (\mathcal{A}, D) + \frac{2R\Gamma^2}{d} \sqrt{\frac{2K}{\delta}\left( \frac{K}{n} + 6K + 6 \right)}. \nonumber
\end{align}
\end{theorem}

\begin{remark}
There are several interesting things to note regarding to this theorem.

First, it explains our intuitive conjecture that a large distortion would incur overfitting during metric learning. It encourages us to choose a small value of $K$ to improve the generalization ability of $\mathcal{A}$. On the other side, setting $K$ as small as possible is unwise, since it would constrain our hypothesis class too much and thus may increase both the \emph{true risk} and \emph{empirical risk}. Therefore, it suggests choosing moderately small values of $K$ in practice via cross validation.

Second, the generalization error tends to decrease with the increase of the dimension $d$. This phenomenon seems to be a bit counter-intuitive since it implies that our method becomes more stable in higher dimensional feature space. However, this does happen only if the previous assumption holds, \ie, $M$ is full rank. In this case, the increase of the dimension squash the range of the spectrum of $M$ since the trace bound $R$ and distortion bound $K$ are fixed. If $M$ is instead rank-deficient, the above analysis does not stand. In particular, the bound will have a dependency on the rank of $M$ and its perturbation. This means that naively increasing the feature dimension by adding zero will not make the bound tighter. We will illustrate this issue in the appendix.
\end{remark}


\section{Experiments}\label{sect:experiemtns}

\begin{table*}[t]
\begin{center}
\begin{tabular}{@{}c@{\hskip 0.05in}|@{\hskip 0.05in}c@{\hskip 0.05in}|@{\hskip 0.05in}c@{\hskip 0.05in}|@{\hskip 0.05in}c@{\hskip 0.05in}|@{\hskip 0.05in}c@{\hskip 0.05in}|@{\hskip 0.05in}c@{}}
\hline
Dataset & Wine 			& Iris 			& Diabetes 			& Waveform 	& Segment \\ \hline \hline
Euc & 3.46(3.60) 		&  5.11(2.58) 		& 31.09(2.03) 			&  18.87(0.65) 	& 5.61(0.92) \\ \hline
Xing & 4.04(4.00) 		&  6.67(3.11) 		& 32.09(3.56)  			&  16.43(1.00)  	& 5.26(0.60) \\ \hline
LMNN & 3.08(2.07)  		&  4.22(1.95) 		& 29.70(3.20)  			&  18.61(0.72) 	& 3.69(0.70) \\ \hline
ITML & \textbf{1.15}(2.07) 	&  4.44(2.57)  		& 29.96(2.97) 			&  15.94(0.83) & 5.02(0.70) \\ \hline
BoostMetric & 2.31(2.18) 		&  3.56(2.52) 		& 26.78(2.12) 			&  16.86(0.90) & 4.21(0.48) \\ \hline
\textit{p}-BDML & 2.83(1.3) 		&  3.11(2.61) 		& 27.57(2.21) 			&  15.78(0.60) & 4.21(0.79) \\ \hline
\textit{t}-BDML 	& 2.26(2.30) 		&  \textbf{2.44}(1.64)	& \textbf{26.43}(2.30) 	&  \textbf{15.34}(0.72) &  \textbf{3.62}(0.34)
\\ \hline
\end{tabular}
\end{center}
\vspace{-6pt}
\caption[Comparison on UCI datasets.]{Comparison of average test errors (\%) and standard deviations on UCI datasets.}
\label{tbl:UCI}
\end{table*}

We present empirical evaluations of our BDML algorithm on a wide range of tasks, including classification on several UCI datasets~\cite{BacheLichman2013}, domain adaptation on medium-scale datasets~\cite{saenko2010adapting}, and face verification on the large-scale LFW dataset~\cite{LFWTech}. Before presenting the results, we first discuss a practical strategy to speed up the bisection method, since it is sometimes hard to estimate a tight interval of the optimal objective value in advance. Specifically, we select several fixed upper bounds and then solve the convex feasibility problem~(\ref{eq:feasibilityProb}) in parallel. If the trial is successful, we use it to shrink the upper bound, otherwise we shrink the lower bound. This procedure provides us a largely reduced interval with time cost as small as one call of MWU solver. Note that we set parameters via cross validation. The impact of different parameters and runtime are provided in the appendix due to space limits.

\subsection{Classification}\label{sect:classification}

We first conduct classification experiments on several UCI datasets, including Wine, Iris, Diabetes, Segment and Waveform, to validate the effectiveness of our BDML. We randomly split datasets into 70\% for training and 30\% for testing and report the average test errors and standard deviations by repeating the random splits for $10$ times. We compare with the baseline of Euclidean metric and several strong competitors like Xing~\cite{xing2003distance}, LMNN~\cite{weinberger2009distance}, ITML~\cite{davis2007information} and BoostMetric~\cite{shen2009positive}. The neighborhood size of kNN classifier is $3$ and all metric $M$ are initialized as the identity matrix. We carefully set other parameters for these methods via cross validation. The results are listed in Table.~\ref{tbl:UCI}, in which the best ones are bolded. Both our \textit{p}-BDML and \textit{t}-BDML perform well on these datasets. Especially, \textit{t}-BDML is consistently better than \textit{p}-BDML which validates the effectiveness of triplet constraints as suggested by~\cite{weinberger2009distance}.

We then demonstrate how the performance of the kNN classifier varies according to the condition number of the learned metric in Fig.~\ref{fig:ValidationDistortion}. The x-axis is the natural logarithm of the condition number. It is clear from the figure that, the average test errors of both \textit{p}-BDML and \textit{t}-BDML first decrease and then increase when the condition numbers become larger. These results provide strong evidence to support our previous analysis that a largely distorted metric space leads to overfitting and a small distortion may result in underfitting. And our BDML effectively controls the distortion, thus improving the generalization ability.

\begin{figure}[t]
  \centering
    \begin{tabular}{@{\hspace{1mm}}c@{\hspace{0mm}}c@{\hspace{1mm}}}
    {\includegraphics[width=0.5\linewidth]{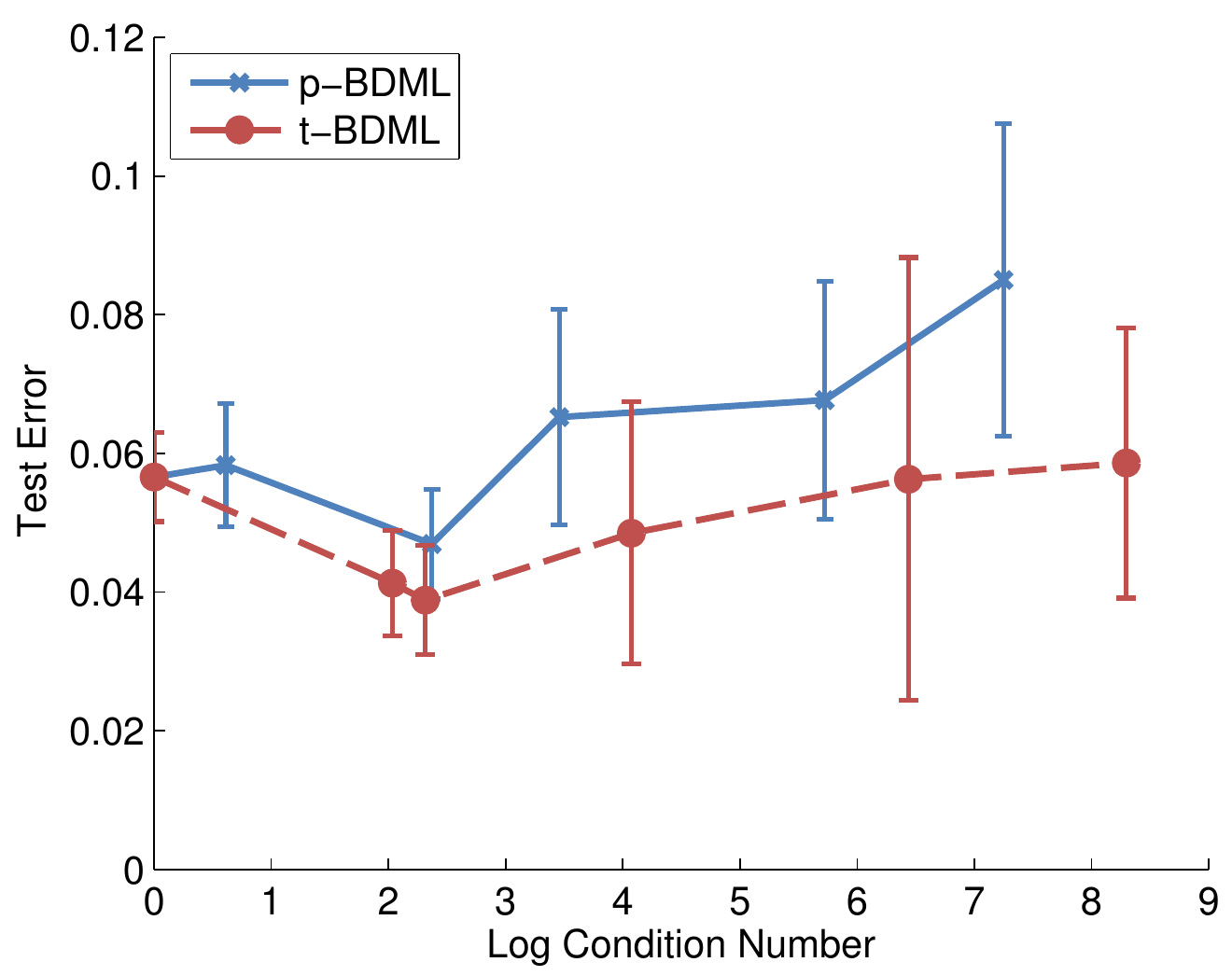}} &
    {\includegraphics[width=0.5\linewidth]{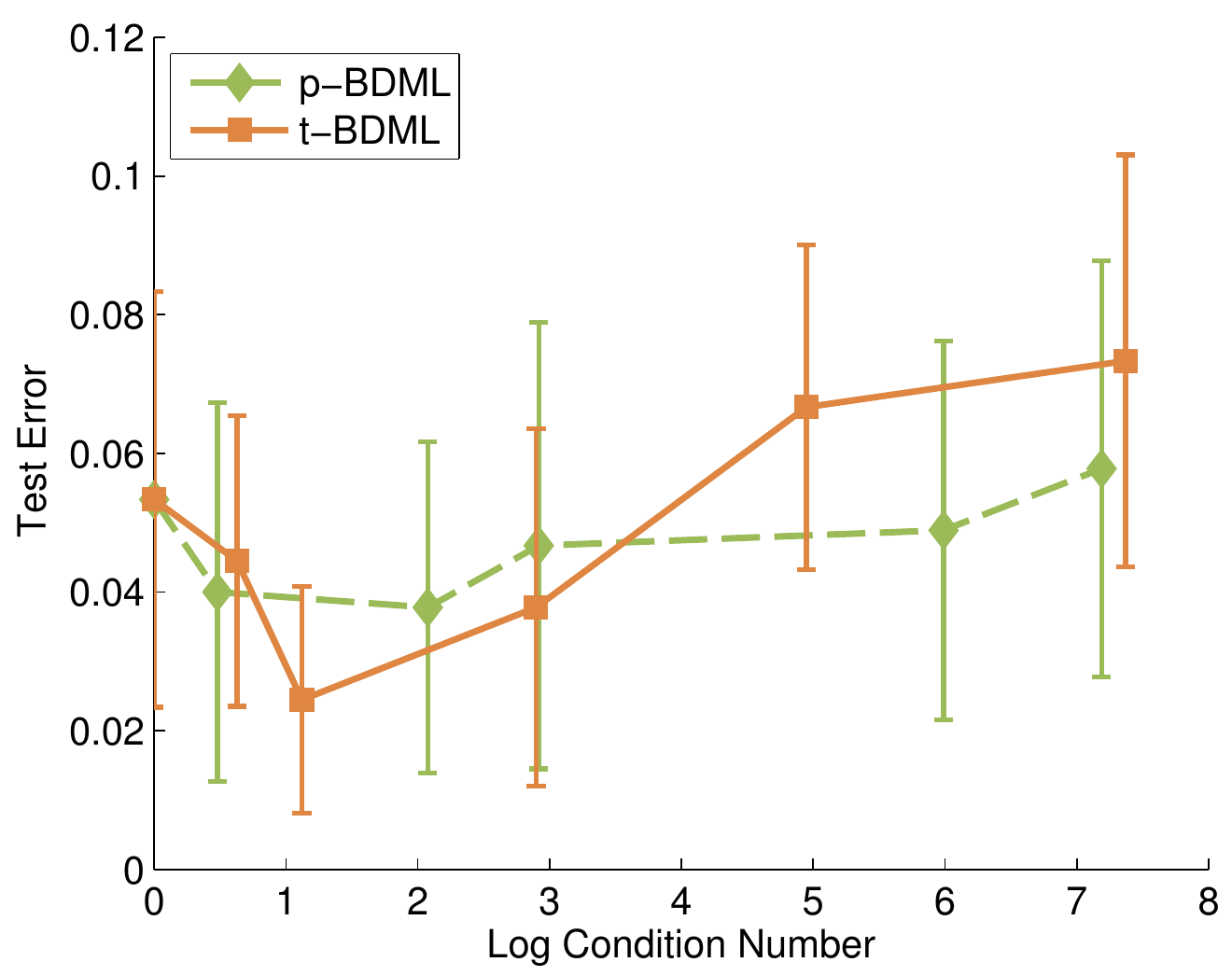}}
  \vspace{-4pt}\\
    {\small (a)} & {\small (b)}
    \end{tabular}
  \vspace{-8pt}
  \caption{Average test error varies with the natural logarithm of condition number on (a) Segment and (b) Iris datasets.}
  \label{fig:ValidationDistortion}
\end{figure}

\subsection{Domain Adaptation}\label{subsect:DomainAdaptation}

We also apply our BDML to domain adaptation problems, under both unsupervised and semi-supervised settings. In the former case, we labeled samples in a source domain for training and want to test the unlabeled samples in the target domain. While in the later setting, apart from labeled samples in a source domain, a small number of labeled samples in the target domain are also accessible during training. We use the same dataset as in~\cite{saenko2010adapting}, which contains 2,533 images of 10 categories from $4$ domains: \emph{\bf C}altech, \emph{\bf A}mazon, \emph{\bf W}ebcam, and \emph{\bf D}slr. We exploit the same 10 categories as~\cite{gong2012geodesic} to all the four domains. Experiments are repeated with 20 fixed train/test splits offered by~\cite{saenko2010adapting}. We set the number of neighbors of kNN to 1 as other methods. Since the original SURF feature is of 800 dimension, we perform pseudometric learning with \textit{t}-BDML and initialize the dimension-reduction mapping via PCA. The size of mapping matrix is set to $30 \times 800$ according to cross-validation.

Table~\ref{tbl:DomainAdaptation} presents the mean test accuracy and standard errors of various metric learning based methods, including LMNN~\cite{weinberger2009distance}, ITML~\cite{saenko2010adapting}, SGF~\cite{gopalan2011domain}, GFK~\cite{gong2012geodesic} and DML-eig~\cite{ying2012distance}, where A $\!\rightarrow\!$ C means the adaptation from source domain A (\ie, Amazon) to the target domain C (\ie, Caltech). And for fair competition, we adopt the best results of GFK under the PCA subspace setting reported by~\cite{gong2012geodesic}. The results in the original Euclidean space are denoted as OrigFeat. In most subtasks of these two settings, our \textit{t}-BDML outperforms than other competitors which demonstrates the strength of the proposed pseudo-metric learning scheme.

\begin{table*}
\begin{center}
{\scriptsize
\begin{tabular}{c|c|c|c|c|c|c|c|c}
\hline
  Methods & A $\rightarrow$ C & A $\rightarrow$ W & C $\rightarrow$ A & C $\rightarrow$ D & D $\rightarrow$ A & D $\rightarrow$ W & W $\rightarrow$ A & W $\rightarrow$ D \\ \hline \hline
  OrigFeat & 22.6(0.3) & 23.5(0.6) & 20.8(0.4) & 22.0(0.6) & 27.7(0.4) & 53.1(0.6) & 20.7(0.6) & 37.3(1.2) \\ \hline
  SGF & 35.3(0.5) & 31.0(0.7) & 36.8(0.5) & 32.6(0.8) & 32.0(0.4) & 66.0(0.5) & 27.5(0.5) & 54.3(1.2) \\ \hline
  GFK & 35.6(0.4) & 34.4(0.9) & \textbf{36.9}(0.4) & \textbf{35.2}(1.0) & 32.5(0.5) & 74.9(0.6) & 31.1(0.8) & 70.6(0.9) \\ \hline
  LMNN & 35.7(0.5) & 32.9(0.8) & 33.8(0.7) & 31.5(1.6) & 33.7(0.4) & 75.1(0.8) & 30.8(0.7) & 67.6(1.0) \\ \hline
  DML-eig & 35.0(0.7) & 28.9(0.7) & 33.7(0.7) & 32.7(1.3) & 33.4(0.3) & 78.2(0.8) & 32.5(0.9) & 72.4(0.6) \\ \hline
  \textit{t}-BDML & \textbf{37.2}(0.3) & \textbf{35.2}(1.0) & 35.2(0.7) & 33.4(1.3) & \textbf{37.1}(0.6) & \textbf{78.6}(0.7) & \textbf{33.2}(0.8) & \textbf{73.8}(0.6) \\ \hline
  \hline
  OrigFeat(semi) & 24.0(0.3) & 31.6(0.6) & 23.1(0.4) & 26.5(0.7) & 31.3(0.7) & 55.5(0.7) & 30.8(0.6) & 44.3(1.0) \\ \hline
  ITML(semi) & 27.3(0.7) & 36.0(1.0) & 33.7(0.8) & 35.0(1.1) & 30.3(0.8) & 55.6(0.7) & 32.3(0.8) & 51.3(0.9) \\ \hline
  SGF(semi) & 37.7(0.5) & 37.9(0.7) & 40.2(0.7) & 36.6(0.8) & 39.2(0.7) & 69.5(0.9) & 38.2(0.6) & 60.6(1.0) \\ \hline
  GFK(semi) & 37.8(0.4) & 53.7(0.8) & 42.0(0.5) & 49.5(0.9) & \textbf{45.0}(0.7) & 78.7(0.5) & 42.8(0.7) & 75.0(0.7) \\ \hline
  LMNN(semi) & 36.6(0.6) & 49.6(0.9) & 43.3(0.5) & 50.3(1.3) & 42.0(0.7) & 78.6(0.7) & 42.3(0.6) & 72.8(1.1) \\ \hline
  DML-eig (semi) & 27.8(0.7) & 40.5(1.0) & 43.3(0.6) & 45.1(1.6) & 43.4(0.6) & 80.5(0.9) & 40.8(0.7) & 76.8(0.9)\\ \hline
  \textit{t}-BDML(semi) & \textbf{38.8}(0.3) & \textbf{55.8}(1.1) & \textbf{44.6}(0.6) & \textbf{54.0}(1.1) & 43.9(0.6) & \textbf{83.8}(0.5) & \textbf{44.8}(0.6) & \textbf{79.2}(0.7) \\ \hline
\end{tabular}}
\end{center}
\vspace{-4pt}
\caption{Comparison of unsupervised (upper part) and semi-supervised (lower part denoted with ``semi'') domain adaptation. Mean test accuracy (\%) and standard error (inside parentheses) are reported.}
\label{tbl:DomainAdaptation}
\end{table*}

\subsection{Face Verification}\label{subsect:FaceVerification}

Finally, we apply our BDML to an unconstrained face verification task, using the large-scale LFW dataset that contains 13,233 face images of 5,749 people. It is  challenging due to the large variations of faces in illumination, expression, pose, resolution, etc. There are 6 standard protocols~\cite{LFWTech} for evaluating results. We use the setting called ``Image-Restricted, Label-Free Outside Data", where we can only access the provided labeled pairs of faces during training. Thus we only compare pseudo-metric learning of \textit{p}-BDML since using triplet constraints would violate this setting. The dateset is organized in 10 folders and each of them contains 300 similar pairs of faces and 300 dissimilar ones. The reported accuracy is obtained via cross validation on the provided 10 folds.

Current state-of-the-art methods under this setting often build various classifiers and combine multiple types of visual descriptors. Since we primarily aim at validating the effectiveness of BDML, we do not carry out intensive feature engineering or build complex similarity measurements. Instead, we use the public ``funneled" SIFT feature\footnote{\burl{http://lear.inrialpes.fr/people/guillaumin/data.php}}
 and regard the learned distance metric as the similarity measure. Since the dimension of raw SIFT feature is too large ($\approx$ 4k), we reduce it to 800 via PCA before pseudo-metric learning with \textit{p}-BDML and set the size of dimension-reduction mapping as $300 \times 800$.

In Table~\ref{tbl:LFW}, we present results of \textit{p}-BDML and other metric learning based methods including Xing~\cite{xing2003distance}, ITML~\cite{davis2007information}, LDML~\cite{guillaumin2009you}, KISSME~\cite{kostinger2012large}, DML-eig~\cite{ying2012distance}, Sub-ITML and Sub-ML~\cite{cao2013similarity}. We also report the average condition number. The ITML methods optimize a logDet regularizer, which yield condition numbers close to one. It suggests that the metric space is not sufficiently distorted for good generalization performance. While the F-norm based regularization methods (\eg, Sub-ML) yield too large distortion, which is not good for generalization. These results support our theoretical analysis in Sec.~\ref{sect:general_bounds} again. In contrast, our BDML method obtains an appropriate condition number, which results in decent generalization performance. Full comparison (including ROC curves) with other non-distance based methods is presented in the appendix.

\addtolength{\tabcolsep}{-1pt}
\begin{table}
\begin{center}
\begin{tabular}{c|c|c|c}
\hline
  Methods 			& Mean Cond. Num. 	& Mean Acc.         	& Std Err.   \\ \hline \hline
  Xing      			& ---               		& 0.7593            	& 0.0059            \\ \hline
  ITML      			& 1.2037            	& 0.7812            	& 0.0045            \\ \hline
  LDML     			& --- 				& 0.7927                	& 0.0060            \\ \hline
  DML-eig   		& 3228.4            	& 0.8127            	& 0.0230            \\ \hline
  Sub-ITML  		& 1.2157            	& 0.8145            	& 0.0046            \\ \hline
  KISSME    		& ---               		& 0.8308            	& 0.0056            \\ \hline
  Sub-ML    		& 1171.397          	& 0.8330            	& 0.0026            \\ \hline
  \textit{p}-BDML      	& 8.3771            	& \textbf{0.8632}   	& \textbf{0.0022}   \\ \hline
\end{tabular}
\end{center}
\caption{Results of face verification on LFW, where ``---" means not applicable due to no public implementation or an unbounded condition number.}
\label{tbl:LFW}
\end{table}

\section{Conclusions}

In this paper, we propose the bounded-distortion metric learning (BDML), which well-balances the fitness to data and the distortion of metric embedding. For Mahalanobis metric space, BDML leads to a bounded condition number metric learning method, which possesses intriguing properties. We propose an efficient learning algorithm and further provide theoretical analysis, which explains why the distortion is a key ingredient to ensure good generalization ability. Also, we generalize to pseudo-metric learning and propose an approximate solver based on the semidefinite relaxation and a randomised algorithm. Empirical results validate that our BDML leads to both better generalization and well-conditionness. In future, we would like to extend the distortion to non-Mahalanobis metric and design corresponding approximation algorithms.

\clearpage

\section{Appendix}\label{sect:appx}

\subsection{Proofs in Section \ref{sect:model}}

\textbf{Proof of Proposition \ref{pro:distortion}}.

\begin{proof}

Note that for any two different points $x$, $y$, we have,
\begin{align}
\frac{{d_M}(f(x),f(y))}{{d_I}(x,y)} = \frac{(x-y)^{\top}M(x-y)}{(x-y)^{\top}(x-y)}. \nonumber
\end{align}
It is easy to see that above equation is the Rayleigh quotient of the PSD matrix $M$. Therefore,
\begin{align}
\lambda_{\min} \le \frac{{d_M}(f(x),f(y))}{{d_I}(x,y)} \le \lambda_{\max}, \nonumber
\end{align}
where $\lambda_{\min}$ and $\lambda_{\max}$ are the minimum and maximum eigenvalue of $M$ respectively. According to the Definition~\ref{embedding}, setting $r = \lambda_{\min}$, we can find that for any $c \ge \frac{\lambda_{\max}}{\lambda_{\min}}$, it is always true,
\begin{align}
r \cdot d_{\mathcal{X}}(x,y) \le d_{\mathcal{Y}}(f(x),f(y)) \le cr \cdot d_{\mathcal{X}}(x,y). \nonumber
\end{align}
Hence the distortion of the metric embedding $f_{I \to M}$ is $\inf\{c | c \ge \frac{\lambda_{\max}}{\lambda_{\min}}\} = \kappa(M)$.

\end{proof}


\subsection{Proofs in Section \ref{sect:solver}}

\textbf{Proof of Lemma \ref{lem:feasibilityEquivalence}}.

\begin{proof}

If $Y^{*}$ is a feasible solution of problem~(\ref{eq:feasibilityProbSDP}), then ${J_i} \bullet Y^{*} \ge {h_i}, \forall i$. Since $\forall i, 0 \le p_i \le 1$, we have $\sum\nolimits_{i = 1}^m {p_i\left( {{J_i} \bullet Y^{*} - {h_i}} \right)}  \ge 0$. Hence, $Y^{*}$ is also a feasible solution of problem~(\ref{eq:feasibilityProbORACLE}).

If there exists a probability vector $\mathbf{p}$ such that the problem~(\ref{eq:feasibilityProbORACLE}) is infeasible, then for all $Y \in \mathbb{P}_{R}^{3d+1}$, we have $\sum\nolimits_{i = 1}^m {p_i\left( {{J_i} \bullet Y - {h_i}} \right)} < 0$. Therefore, the original problem~(\ref{eq:feasibilityProbSDP}) is also infeasible, since otherwise there exists a solution $\tilde Y$ such that $\tilde Y \in \mathbb{P}_{R}^{3d+1}$ and $\sum\nolimits_{i = 1}^m {p_i\left( {{J_i} \bullet {\tilde Y} - {h_i}} \right)} \ge 0$.

\end{proof}


\textbf{Proof of Theorem \ref{the:correctnessMWU}}.

\begin{proof}

In the $t$-th round, we run the \textsc{Oracle} with a probability distribution $p^{(t)}$ as input.

If the \textsc{Oracle} declares the problem~(\ref{eq:feasibilityProbORACLE}) is infeasible, then due to Lemma.~\ref{lem:feasibilityEquivalence}, the original problem is correctly concluded as infeasible.

On the other hand, if this situation never happens during the iteration, i.e., for any round $t$, \textsc{Oracle} succeeds to find a solution $Y^{(t)}$ to problem~(\ref{eq:feasibilityProbORACLE}), then we can get the following inequality, $\forall t = 1, 2, ..., T$,
\begin{align}
\sum\limits_{i = 1}^m {p_i^{(t)}\eta _i^{(t)}}  = \frac{1}{\rho }\sum\limits_{i = 1}^m {p_i^{(t)}\left( {{J_i} \bullet {Y^{(t)}} - {h_i}} \right)} \ge 0, \nonumber
\end{align}
Then, equipped with the above inequality and Theorem 2 in~\cite{kale2007efficient}, we have, for any constraint $i$,
{\small
\begin{align}
0 & \le \sum\limits_{t = 1}^T {\eta _i^{(t)}} + \varepsilon \sum\limits_{t = 1}^T {\left| {\eta _i^{(t)}} \right|}  + \frac{{\ln (m)}}{\varepsilon} \nonumber \\
& = \frac{1}{\rho }\sum\limits_{t = 1}^T {\left( {{J_i} \bullet {Y^{(t)}} - {h_i}} \right)} \nonumber \\
& + \frac{\varepsilon }{\rho }\sum\limits_{t = 1}^T {\left| {{J_i} \bullet {Y^{(t)}} - {h_i}} \right|} + \frac{{\ln (m)}}{\varepsilon } \nonumber \\
& = \frac{{\left( {1 + \varepsilon } \right)}}{\rho }\sum\limits_{t = 1}^T {\left( {{J_i} \bullet {Y^{(t)}} - {h_i}} \right)} \nonumber \\
& + \frac{{2\varepsilon }}{\rho }\sum\limits_{t \in \mathcal{T}_{-}} {\left| {{J_i} \bullet {Y^{(t)}} - {h_i}} \right|}  + \frac{{\ln (m)}}{\varepsilon } \nonumber \\
& \le \frac{{\left( {1 + \varepsilon } \right)}}{\rho }\sum\limits_{t = 1}^T {\left( {{J_i} \bullet {Y^{(t)}} - {h_i}} \right)} + 2\varepsilon T + \frac{{\ln (m)}}{\varepsilon }. \nonumber
\end{align}
}
Here $\mathcal{T}_{-}$ denotes the set of index $t$ when ${{J_i} \bullet {Y^{(t)}} - {h_i}} < 0$. Divided by $T$ on both sides of the above inequality and with some rearrangement, we can obtain,
{\small
\begin{align}
{J_i} \bullet \left( {\frac{1}{T}\sum\limits_{t = 1}^T {{Y^{(t)}}} } \right) - {h_i} \ge  - \frac{\rho }{{1 + \varepsilon }}\left( {2\varepsilon  + \frac{{\ln (m)}}{{\varepsilon T}}} \right). \nonumber
\end{align}
}
Note that $\bar Y = (\sum\nolimits_{t = 1}^T {{Y^{(t)}}} )/T$ is the solution returned by the multiplicative weights update method. Now, if we set $\varepsilon  = \frac{{{C_1}\delta }}{\rho }$, ${T = \frac{{{C_2}{\rho ^2}\ln (m)}}{{{\delta ^2}}}}$, where $C_1, C_2$ are positive real numbers, then we have,
{\small
\begin{align}
- \frac{\rho }{{1 + \varepsilon }}\left( {2\varepsilon  + \frac{{\ln (m)}}{{\varepsilon T}}} \right) =  - \frac{\rho }{{\rho  + {C_1}\delta }}(2{C_1} + \frac{1}{{{C_1}{C_2}}})\delta. \nonumber
\end{align}
}
With some calculation, we can find that if we choose $0 < C_1 < 1/2$ and set ${{C_2} =  - \frac{1}{{{C_1}(2{C_1} - 1)}}}$, e.g., $C_1 = 1/4, C_2 = 8$, then ${J_i} \bullet \bar Y - {h_i} \ge  -\delta / (1+\frac{\delta}{4\rho}) \ge- \delta$ for any constraint $i$.

\end{proof}


\subsection{Proofs in Section \ref{sect:pseudo_metric}}

\textbf{Proof of Proposition \ref{pro:mapping}}.

\begin{proof}

We first restate problem \eqref{eq:BDML_pseudo_metric} as below,
\begin{subequations}
\begin{align}
{\mathop {\min }\limits_{Q \in \mathbb{R}^{q \times d}} } \qquad & { \frac{1}{n} \sum\nolimits_{(i,j) \in \mathcal{S}} {{X_{ij}} \bullet (Q^{\top} \Lambda Q)}} \nonumber \\
{s.t.} \qquad & {{{X_{ij}} \bullet (Q^{\top} \Lambda Q)} \ge \mu}, \qquad {\forall (i,j) \in \mathcal{I}},  \label{margin_cons} \\
& QQ^{\top} = I. \label{orthogonal_cons}
\end{align}
\end{subequations}

We can denote the vectorization of $Q$ as $\zeta$ and rewrite the above problem as the following standard formulation of quadratic constrained quadratic programming (QCQP),
\begin{subequations}
\begin{align}
{\mathop {\min }\limits_{\zeta \in \mathbb{R}^{qd \times 1}} } \qquad & {\frac{1}{n} \sum\nolimits_{(i,j) \in \mathcal{S}} {\zeta^{\top}\tilde{X}_{ij}\zeta}} \nonumber  \\
{s.t.} \qquad & {\zeta^{\top}\tilde{X}_{ij}\zeta} \ge \mu, \qquad & {\forall (i,j) \in \mathcal{I}}, \label{margin_cons_stand} \\
& {\zeta^{\top}A_{uv}\zeta} = b_{uv}, & {\forall (u, v) \in \mathcal{C}}, \label{orthogonal_cons_stand}
\end{align}
\end{subequations}
where the set of margin constraints in \eqref{margin_cons} corresponds to the one in \eqref{margin_cons_stand} and the set of orthogonal constraints in \eqref{orthogonal_cons} corresponds to the one in \eqref{orthogonal_cons_stand}. Specifically, $\tilde{X}_{ij} = X_{ij} \otimes \Lambda$ and $\otimes$ stands for Kronecker product. And we denote the index set of the upper triangular part of a $q$-dimensional squared matrix as $\mathcal{C} = \{(u, v) \in [q] \times [q] | u \le v \}$ where $[q] = \{1, 2, \dots, q\}$. Then for each element ($u, v$) of $\mathcal{C}$, we have $A_{uv}$ is a block diagonal matrix which contains $d$ identical blocks $B_{uv} \in \mathbb{R}^{q \times q}$ as following,
{\small
\begin{align}
A_{uv} = \left( {\begin{array}{*{20}{c}}
{B_{uv}}&{}&{}&{}\\
{}&{B_{uv}}&{}&{}\\
{}&{}&{\ddots}&{}\\
{}&{}&{}&{B_{uv}}
\end{array}} \right) \nonumber
\end{align}
}
where
{\small
\begin{align}
&~~~~~~~~~~~~~~~u~~~~v \nonumber\\
{B_{uv}} = &\left( {\begin{array}{*{20}{c}}
{}&{}&{\vdots}&{}&{}&{}&{}\\
{\dots}&{\vdots}&{1}&{\dots}&{}&{}&{}\\
{\dots}&{1}&{\vdots}&{\dots}&{}&{}&{}\\
{}&{\vdots}&{}&{}&{}&{}&{}\\
{}&{}&{}&{}&{}&{}&{}
\end{array}} \right)
\begin{array}{*{20}{c}}
\vspace{-0.15cm}
{}\\
{}\\
{u}\\
{v}\\
{}\\
{}\\
\end{array} \nonumber
\end{align}
}
\ie, $(u, v)$ and $(v, u)$-th entries of $B$ are 1 while others are 0. And $b_{uv} = 2$ if $u = v$, otherwise $b_{uv} = 0$.

Based on~\cite{luo2010semidefinite}, to derive the SDP relaxation, we can easily observe the following,
\begin{align}
\zeta^{\top}\tilde{X}_{ij}\zeta = \mathbf{Tr}(\tilde{X}_{ij}\zeta\zeta^{\top}) = \tilde{X}_{ij} \bullet \zeta\zeta^{\top}. \nonumber
\end{align}
Thus setting $\tilde{Q} = \zeta\zeta^{\top}$, we can write the equivalent form of the above QCQP as below,
\begin{align}
{\mathop {\min }\limits_{\zeta \in \mathbb{R}^{qd \times 1}} } \qquad & {\frac{1}{n} \sum\nolimits_{(i,j) \in \mathcal{S}} {\tilde{X}_{ij} \bullet \tilde{Q}}} \nonumber \\
{s.t.} \qquad & {\tilde{X}_{ij} \bullet \tilde{Q}} \ge \mu, \qquad & {\forall (i,j) \in \mathcal{I}}, \nonumber \\
& {A_{uv} \bullet \tilde{Q}} = b_{uv}, & {\forall (u, v) \in \mathcal{C}}, \nonumber \\
& rank(\tilde{Q}) = 1. \nonumber
\end{align}
By removing the last rank constraint, we can obtain the desired SDP relaxation.

\end{proof}

\begin{remark}
Note that the orthogonal constraints in the original QCQP problem imply that $\mathbf{Tr}(\tilde{Q}) = q$ in the SDP relaxation problem.
\end{remark}


Before obtaining our \textbf{Theorem \ref{thm:tailbound}}, we need to introduce three lemmas as below. First, we state the Lemma 1 in \cite{luo2007approximation} as below which gives the polynomial tail bound of the left-side inequality $\xi^{\top} H \xi < \gamma \mathbb{E}[\xi^{\top} H \xi]$. Reader can refer to the paper for details of the proof.

\begin{lemma}[\textbf{Left-side Polynomial Tail Bound}]
Let $H \in \mathbb{S}_{+}^{d}, Z \in \mathbb{S}_{+}^{d}$. Suppose $\xi \in \mathbb{R}^d$ is a random vector generated from the real-valued normal distribution $\mathcal{N}(0, Z)$. Then for any $\gamma > 0$,
\begin{align}
Prob \left( \xi^{\top} H \xi < \gamma \mathbb{E}[\xi^{\top} H \xi] \right) < \max \{ \sqrt{\gamma}, \frac{2(\bar{r}-1)\gamma}{\pi - 2} \}, \nonumber
\end{align}
where $\bar{r} = \min \{rank(H), rank(Z)\}$.
\end{lemma}


Then we derive a right-side exponential tail bound as below.

\begin{lemma}[\textbf{Right-side Exponential Tail Bound}]
Let $H \in \mathbb{S}_{+}^{d}, Z \in \mathbb{S}_{+}^{d}$. Suppose $\xi \in \mathbb{R}^d$ is a random vector generated from the real-valued normal distribution $\mathcal{N}(0, Z)$. Then for any $\gamma \ge 1$,
\begin{align}
Prob \left( \xi^{\top} H \xi > \gamma \mathbb{E}[\xi^{\top} H \xi] \right) < \bar{r} \exp{\left(-\frac{1}{2}\left(\gamma-\sqrt{2\gamma - 1}\right)\right)}, \nonumber
\end{align}
where $\bar{r} = \min \{rank(H), rank(Z)\}$.
\end{lemma}

\begin{proof}

Let $r = rank(Z)$. Since $Z \in \mathbb{S}_{+}^{d}$, then we can write $Z = UU^{\top}$ for some $U \in \mathbb{R}^{d \times r}$. We consider the eigen decomposition of the real symmetric matrix $U^{\top}HU = L D L^{\top} = \sum_{i = 1}^{r} \lambda_i l_i l_i^{\top}$, where $L = \left[ l_1, l_2, \dots, l_r \right] \in \mathbb{R}^{r \times r}$ is an orthogonal matrix and $D = diag \{\lambda_1, \lambda_2, \dots, \lambda_r \}$ with $\lambda_1 \ge \lambda_2 \ge \cdots \ge \lambda_r \ge 0$. Note that $\lambda_i = 0$ for all $i > \bar{r}$ due to the fact that $U^{\top}HU$ has rank at most $\bar{r}$. Denoting $\bar{\xi} \sim \mathcal{N}(0, I_r)$, then we can easily check that $U\bar{\xi} \sim \mathcal{N}(0, Z)$, \ie, $U\bar{\xi}$ is statistically identical to $\xi$.

Therefore, we have that,
\begin{align}
Prob \left( \xi^{\top} H \xi > \gamma \mathbb{E}[\xi^{\top} H \xi] \right) & = Prob \left( \bar{\xi}^{\top} U^{\top} H U \bar{\xi} > \gamma \mathbb{E}[\bar{\xi}^{\top} U^{\top} H U \bar{\xi}] \right) \nonumber \\
& = Prob \left( \sum_{i=1}^{\bar{r}} \lambda_i {(l_i^{\top}\bar{\xi})^2} > \gamma \mathbb{E}[\sum_{i=1}^{\bar{r}} \lambda_i {(l_i^{\top}\bar{\xi})^2}] \right) \nonumber
\end{align}
Denoting $u_i = l_i^{\top}\bar{\xi}$, we have that $\mathbb{E}[u_i] = 0$ and $\mathbb{E}[u_i^2] = 1$, \ie, $u_i$ is a standard normal variable. Hence,
\begin{align}
Prob \left( \xi^{\top} H \xi > \gamma \mathbb{E}[\xi^{\top} H \xi] \right) & = Prob \left( \sum_{i=1}^{\bar{r}} \lambda_i{u_i^2} > \gamma \mathbb{E}[\sum_{i=1}^{\bar{r}} \lambda_i{u_i^2}] \right) \nonumber \\
& = Prob \left( \sum_{i=1}^{\bar{r}} \lambda_i{u_i^2} > \gamma \sum_{i=1}^{\bar{r}} \lambda_i \right) \nonumber \\
& = Prob \left( \sum_{i=1}^{\bar{r}} \bar{\lambda}_i{u_i^2} > \gamma \right), \nonumber
\end{align}
where $\bar{\lambda}_i = {\lambda_i}/{\sum_{i=1}^{\bar{r}} \lambda_i}$ for $i = 1,\dots,\bar{r}$. Note that $\sum_{i=1}^{\bar{r}} \bar{\lambda}_i = 1$ and $\bar{\lambda}_1 \ge \bar{\lambda}_2 \ge \cdots \ge \bar{\lambda}_r \ge 0$.
Then
\begin{align}
Prob \left( \sum_{i=1}^{\bar{r}} \bar{\lambda}_i{u_i^2} > \gamma \right) & = 1 - Prob \left( \sum_{i=1}^{\bar{r}} \bar{\lambda}_i{u_i^2} \le \gamma \right) \nonumber \\
& \le 1 - Prob \left( {u_1^2} \le \gamma ~~ \& ~~ \dots ~~ \& ~~ {u_{\bar{r}}^2} \le \gamma \right) \nonumber \\
& = Prob \left( {u_1^2} > \gamma ~~ || ~~ \dots ~~ || ~~ {u_{\bar{r}}^2} > \gamma \right) \nonumber \\
& \le \sum_{i=1}^{\bar{r}} Prob \left( {u_i^2} > \gamma \right) \nonumber \\
& \le \bar{r} \exp{\left(-\frac{1}{2}\left(\gamma-\sqrt{2\gamma - 1}\right)\right)}  \nonumber
\end{align}
Note that the last step is due to the inequality in Lemma (\ref{chi_square_tail_bound_3}) and the fact that ${u_i^2}$ is a $\chi^2$ random variable with 1 degree of freedom.

\end{proof}


At last, we derive another two-side exponential tail bound for our own purpose.

\begin{lemma}[\textbf{Two-side Exponential Tail Bound}]
Let $\tilde{Q}^{*} \in \mathbb{S}_{+}^{qd}$ be the optimal solution of problem \eqref{eq:BDML_pseudometric_relax}, and $A_{uv} \in \mathbb{S}^{qd}$ be any constraint matrix corresponding to index set $\mathcal{C}$ in problem \eqref{eq:BDML_pseudometric_relax}. Suppose $\xi \in \mathbb{R}^{qd}$ is a random vector generated from the real-valued normal distribution $\mathcal{N}(0, \tilde{Q}^{*})$. Then for any $\epsilon \ge 0$
\begin{align}
Prob \left( |\xi^{\top} A_{uv} \xi - \mathbb{E}[\xi^{\top} A_{uv} \xi]| \ge \epsilon \right)
\le \bar{r} \left[ \exp{\left(-\frac{(\tau - 1)^2}{4}\right)} + \exp{\left(-\frac{\epsilon^2}{8\bar{r}dq^2}\right)} \right], \nonumber
\end{align}
where $\bar{r} = \min \{rank(A_{uv}), rank(\tilde{Q}^{*})\}$ and $\tau = \left(\frac{\epsilon}{q}\sqrt{\frac{2}{\bar{r}d}} + 1\right)^{\frac{1}{2}}$.
\end{lemma}

\begin{proof}

Let $r = rank(\tilde{Q}^{*})$. Since $\tilde{Q}^{*} \in \mathbb{S}_{+}^{qd}$, we can write $\tilde{Q}^{*} = UU^{\top}$ for some $U \in \mathbb{R}^{qd \times r}$. And since $A_{uv}$ is symmetric, we can write the eigen-decomposition of matrix $U^{\top}A_{uv}U = LDL^{\top}$, where $L = \left[ l_1, l_2, \dots, l_r \right] \in \mathbb{R}^{r \times r}$ is an orthogonal matrix and $D = diag \{\lambda_1, \lambda_2, \dots, \lambda_r \}$ with $\lambda_1 \ge \lambda_2 \ge \cdots \ge \lambda_r \ge 0$. Note that $\lambda_i = 0$ for all $i > \bar{r}$ due to the fact that $U^{\top}HU$ has rank at most $\bar{r}$. Denoting $\bar{\xi} \sim \mathcal{N}(0, I_r)$, then we can easily check that $U\bar{\xi} \sim \mathcal{N}(0, Z)$, \ie, $U\bar{\xi}$ is statistically identical to $\xi$.

Therefore, we have that,
\begin{align}
Prob \left( |\xi^{\top} A_{uv} \xi - \mathbb{E}[\xi^{\top} A_{uv} \xi]| \ge \epsilon \right)
= Prob \left( |\sum_{i=1}^{\bar{r}} \lambda_i {(l_i^{\top}\bar{\xi})^2} - \sum_{i=1}^{\bar{r}} \mathbb{E}[ \lambda_i {(l_i^{\top}\bar{\xi})^2} ] | \ge \epsilon \right) \nonumber
\end{align}
Denoting $u_i = l_i^{\top}\bar{\xi}$, we have that $\mathbb{E}[u_i] = 0$ and $\mathbb{E}[u_i^2] = 1$, \ie, $u_i$ is a standard normal variable. Hence,
\begin{align}
Prob \left( |\xi^{\top} A_{uv} \xi - \mathbb{E}[\xi^{\top} A_{uv} \xi]| \ge \epsilon \right) = & Prob \left( |\sum_{i=1}^{\bar{r}} \lambda_i (u_i^2 - 1)| \ge \epsilon \right) \nonumber \\
\le &~ Prob \left( \sqrt{\sum_{i=1}^{\bar{r}} \lambda_i^2} \sqrt{\sum_{i=1}^{\bar{r}}(u_i^2 - 1)^2} \ge \epsilon \right) \nonumber \\
= & Prob \left( ||D||_{F} \sqrt{\sum_{i=1}^{\bar{r}}(u_i^2 - 1)^2} \ge \epsilon \right) \nonumber
\end{align}
Note that,
\begin{align}
||D||_{F} & = ||U^{\top}A_{uv}U||_{F} \le ||A_{uv}||_{F}||U||_{F}^{2} = ||A_{uv}||_{F}\mathbf{Tr}(\tilde{Q}^{*}) \le q\sqrt{2d}. \nonumber
\end{align}
Here the first equality uses the fact that Frobenius norm is rotation-invariant. The second inequality uses the Cauchy-Schwarz inequality. While for the last equality, we can see that if $u = v$, then $||A_{uv}||_{F} = \sqrt{d}$, otherwise $||A_{uv}||_{F} = \sqrt{2d}$. Moreover, we have $\mathbf{Tr}(\tilde{Q}^{*}) = q$ due to the fact $\tilde{Q}^{*}$ is the optimal solution of problem \eqref{eq:BDML_pseudometric_relax} satisfying the orthogonal constraints. Therefore, we have that,
\begin{align}
Prob \left( |\xi^{\top} A_{uv} \xi - \mathbb{E}[\xi^{\top} A_{uv} \xi]| \ge \epsilon \right) \le & Prob \left( \sum_{i=1}^{\bar{r}}({u_i^2} - 1)^2 \ge \frac{\epsilon^2}{2dq^2} \right) \nonumber \\
= & 1 - Prob \left( \sum_{i=1}^{\bar{r}}({u_i^2} - 1)^2 < \frac{\epsilon^2}{2dq^2} \right) \nonumber \\
\le & 1 - Prob \left( ({u_1^2} - 1)^2 < \frac{\epsilon^2}{2\bar{r}dq^2} ~ \& ~ \dots ~ \& ~ ({u_{\bar{r}}^2} - 1)^2 < \frac{\epsilon^2}{2\bar{r}dq^2} \right) \nonumber \\
= & Prob \left( ({u_1^2} - 1)^2 \ge \frac{\epsilon^2}{2\bar{r}dq^2} ~ || ~ \dots ~ || ~ ({u_{\bar{r}}^2} - 1)^2 \ge \frac{\epsilon^2}{2\bar{r}dq^2} \right) \nonumber \\
\le & \sum_{i=1}^{\bar{r}} Prob \left( ({u_i^2} - 1)^2 \ge \frac{\epsilon^2}{2\bar{r}dq^2} \right) \nonumber \\
= & \sum_{i=1}^{\bar{r}} Prob \left( |{u_i^2} - 1| \ge \frac{\epsilon}{q\sqrt{2\bar{r}d}} \right) \nonumber \\
= & \sum_{i=1}^{\bar{r}} \left[ Prob \left( {u_i^2} \ge 1 + \frac{\epsilon}{q\sqrt{2\bar{r}d}} \right) +Prob \left( {u_i^2} \le 1 - \frac{\epsilon}{q\sqrt{2\bar{r}d}} \right) \right] \nonumber \\
\le & \bar{r} \left[ \exp{\left(-\frac{(\tau - 1)^2}{4}\right)} + \exp{\left(-\frac{\epsilon^2}{8\bar{r}dq^2}\right)} \right]. \nonumber
\end{align}
where $\tau = \left(\frac{\epsilon}{q}\sqrt{\frac{2}{\bar{r}d}} + 1\right)^{\frac{1}{2}}$. In the last step, we use the exponential tail bound of Chi-square variable in Lemma \ref{chi_square_tail_bound_3}.

\end{proof}

\begin{remark}
Note that if $u = v$ then $rank(A_{uv}) = d$ otherwise $rank(A_{uv}) = 2d$. Thus, we have that $\bar{r} \le 2d$, which could be used to eliminate the variable $\bar{r}$ in the tail bound.
\end{remark}


\textbf{Proof of Theorem \ref{thm:tailbound}}.

\begin{proof}

First, We have the following,
{\small
\begin{align}
& Prob \left( \nu \ge \gamma \mu ~~  \& ~~ \zeta \le \epsilon ~~  \& ~~ \xi^{\top} \tilde{G} \xi \le \omega \mathbb{E}[\xi^{\top} \tilde{G} \xi] \right) \nonumber \\
& \ge 1 - Prob \left( \exists (i,j) ~~ \xi^{\top} \tilde{X}_{ij} \xi < \gamma \mu \right) - Prob \left( \exists {(u,v)} ~~ |\xi^{\top} A_{uv} \xi - b_{uv}| > \epsilon \right) - Prob \left( \xi^{\top} \tilde{G} \xi > \omega \mathbb{E}[\xi^{\top} \tilde{G} \xi] \right) \nonumber \\
& \ge 1 - \sum_{(i,j) \in \mathcal{I}} Prob \left( \xi^{\top} \tilde{X}_{ij} \xi < \gamma \mu \right) - \sum_{(u,v) \in \mathcal{C}} Prob \left( |\xi^{\top} A_{uv} \xi - b_{uv}| > \epsilon \right) - Prob \left( \xi^{\top} \tilde{G} \xi > \omega \mathbb{E}[\xi^{\top} \tilde{G} \xi] \right) \nonumber \\
& = 1 - |\mathcal{I}| + \sum_{(i,j) \in \mathcal{I}} Prob \left( \xi^{\top} \tilde{X}_{ij} \xi \ge \gamma \mu \right) - \sum_{(u,v) \in \mathcal{C}} Prob \left( |\xi^{\top} A_{uv} \xi - b_{uv}| > \epsilon \right) - Prob \left( \xi^{\top} \tilde{G} \xi > \omega \mathbb{E}[\xi^{\top} \tilde{G} \xi] \right). \nonumber
\end{align}}
Here we use the fact that ,
\begin{align}
& \mathbb{E}[\xi^{\top} \tilde{X}_{ij} \xi] = \tilde{X}_{ij} \bullet \tilde{Q}^{*} \ge \mu \nonumber \\
& \mathbb{E}[\xi^{\top} A_{uv} \xi] = A_{uv} \bullet \tilde{Q}^{*} = b_{uv}. \nonumber
\end{align}
We denote
\begin{align}
\mathbb{T}_1 & = \sum_{(i,j) \in \mathcal{I}} Prob \left( \xi^{\top} \tilde{X}_{ij} \xi \ge \gamma \mu \right), \nonumber \\
\mathbb{T}_2 & = \sum_{(u,v) \in \mathcal{C}} Prob \left( |\xi^{\top} A_{uv} \xi - b_{uv}| > \epsilon \right) \nonumber \\
\mathbb{T}_3 & = Prob \left( \xi^{\top} \tilde{G} \xi > \omega \mathbb{E}[\xi^{\top} \tilde{G} \xi] \right). \nonumber
\end{align}
and
\begin{align}
r_1 & = \min\{\max_{(i,j)} ~~ rank(\tilde{X}_{ij}), rank(\tilde{Q}^{*})\}, \nonumber \\
r_2 & = \min\{\max_{(u,v)} ~~ rank(A_{uv}), rank(\tilde{Q}^{*})\},  \nonumber \\
r_3 & = \min\{rank(\tilde{G}), rank(\tilde{Q}^{*})\}. \nonumber
\end{align}
According to the first constraints in \eqref{eq:BDML_pseudometric_relax} and the lemma of left-side polynomial tail bound, we have,
\begin{align}
\mathbb{T}_1 & \ge  \sum_{(i,j) \in \mathcal{I}} Prob \left( \xi^{\top} \tilde{X}_{ij} \xi \ge \gamma \mathbb{E}[\xi^{\top} \tilde{X}_{ij} \xi] \right)  \nonumber \\
& \ge |\mathcal{I}| \left(1 - \max \{ \sqrt{\gamma}, \frac{2(r_1-1)\gamma}{\pi - 2} \} \right). \nonumber
\end{align}
Then according the second constraints in \eqref{eq:BDML_pseudometric_relax} and the lemma of two-side exponential tail, we have,
\begin{align}
\mathbb{T}_2 &=  \sum_{(u,v) \in \mathcal{C}} Prob \left( |\xi^{\top} A_{uv} \xi - \mathbb{E}[\xi^{\top} A_{uv} \xi ]| > \epsilon \right) \nonumber \\
&\le \frac{{r_2}q(q+1)}{2} \left[ \exp{\left(-\frac{(\tau - 1)^2}{4}\right)} + \exp{\left(-\frac{\epsilon^2}{8{r_2}dq^2}\right)} \right], \nonumber
\end{align}
where $\tau$ = $\left(\frac{\epsilon}{q}\sqrt{\frac{2}{{r_2}d}} + 1\right)^{\frac{1}{2}}$. At last, according to the lemma of right-side exponential tail bound, we have,
\begin{align}
\mathbb{T}_3 \le {r_3} \exp{\left(-\frac{1}{2}\left(\omega-\sqrt{2\omega - 1}\right)\right)}. \nonumber
\end{align}
Therefore, based on all above inequalities and facts that $r_1 \le r$, $r_2 \le r$ and $r_3 \le r$, we can derive that,
\begin{align}
& Prob \left(\nu \ge \gamma \mu ~~~ \& ~~~ \zeta \le \epsilon ~~~ \& ~~~ \xi^{\top} \tilde{G} \xi \le \omega \tilde{G} \bullet \tilde{Q}^{*} \right) \ge 1 - |\mathcal{I}| \max \left( \sqrt{\gamma}, \frac{2(r-1)\gamma}{\pi - 2} \right) \nonumber \\
& - r \exp{\left(-\frac{1}{2}\left(\omega-\sqrt{2\omega - 1}\right)\right)} - \frac{rq(q+1)}{2} \left[ \exp{\left(-\frac{(\tau - 1)^2}{4}\right)} + \exp{\left(-\frac{\epsilon^2}{8rdq^2}\right)} \right].  \nonumber
\end{align}

\end{proof}


\subsection{Proofs in Section \ref{sect:general_bounds}}

\textbf{Proof of Lemma \ref{lem:stability_BDML}}.

\begin{proof}

The given loss function is $\ell(\mathcal{A}, X_{ij}) = M \bullet X_{ij}$. First, note that,
\begin{align}
M \bullet X_{ij} & = (x_i - x_j)^{\top} M (x_i - x_j), \nonumber \\
& = \frac{ (x_i - x_j)^{\top} M (x_i - x_j) } {(x_i - x_j)^{\top}(x_i - x_j)} (x_i - x_j)^{\top}(x_i - x_j), \nonumber \nonumber
\end{align}
Then, relying on the property of the Rayleigh quotient, we can obtain that,
\begin{align}
\lambda_{\min} || x_i - x_j ||_{2}^{2} \le M \bullet X_{ij} \le \lambda_{\max} || x_i - x_j ||_{2}^{2}, \nonumber
\end{align}
where $\lambda_{\max}$ and $\lambda_{\min}$ are the maximum and minimum eigenvalues of $M$.

Thus, denoting the replace-one dataset and the corresponding metric as $D^{k}$ and $M^{k}$ respectively, we can derive that,
\begin{align}
\left| \ell(\mathcal{A}_D, X_{ij}) - \ell(\mathcal{A}_{D^{k}}, X_{ij}) \right| & = \left| M \bullet X_{ij} - M^{k} \bullet X_{ij} \right|, \nonumber \\
& \le \left| \lambda_{\max} - \lambda^{k}_{\min} \right| || x_i - x_j ||_{2}^{2} , \nonumber \\
& \le 4\Gamma^2 \left( \lambda_{\max} + \lambda^{k}_{\min} \right) , \nonumber \\
& \le 4\Gamma^2 (\frac{KR}{d} + \frac{R}{d}), \nonumber \\
& \le \frac{4(K+1)R\Gamma^2}{d}, \nonumber
\end{align}
where $\lambda_{\max}$ and $\lambda^{k}_{\min}$ are the maximum and minimum eigenvalues of $M$ and $M^{k}$ respectively. The second inequality uses the fact that $x_i$ and $x_j$ are in a $\Gamma$-ball. And the next one relies on that, for both $M$ and $M^{k}$, it is true that
\begin{align}
\lambda_{\max} \le K \lambda_{\min} \le \frac{K}{d}\mathbf{Tr}(M) \le \frac{KR}{d}. \nonumber
\end{align}
Therefore, the Uniform-Replace-One stability $\beta = \frac{4(K+1)R\Gamma^2}{d}$.

\end{proof}

\begin{remark}

Note if our previous assumption is violated, \ie, $M$ is rank-deficient, then this stability result does not stand any more. In particular, we will have, for any $M$,
\begin{align}
\lambda_{\max} \le K \lambda_{\min} \le \frac{K}{r}\mathbf{Tr}(M) \le \frac{KR}{r}, \nonumber
\end{align}
where $r$ is the rank of $M$ and $r < d$. Then, the stability could be rewritten as $\beta = \frac{4(K+1)R\Gamma^2}{\bar{r}}$, where $\bar{r} = \min \{ rank(M), rank(M^{k}) \}$. This resultant stability is case-dependent, thus being less favourable. This argument is important for our later explanation why our generalization bound will not become tighter and tighter via trivially increasing the feature dimension.

\end{remark}


To prove the Theorem \ref{thm:polybound}, we first state one part of \textbf{Lemma 9} in \cite{bousquet2002stability} as below.

\begin{lemma}[\textbf{Variance Bound}]
For any algorithm $\mathcal{A}$ and loss function $\ell$ such that $0 \le \ell \le L$, we have for any different $i, j \in \{1,\dots,n\}$,
\begin{align}
\mathbb{E}_{D}\left[ \left( \mathcal{R}(\mathcal{A}, D) - \mathcal{R}_{emp} (\mathcal{A}, D) \right)^2 \right] \le \frac{L^2}{2n} + 3L \mathbb{E}_{\{D \bigcup z_i^{\prime}\}}\left[ |\ell(A_D,z_i) - \ell(A_{D^i},z_i)| \right]. \nonumber
\end{align}
\end{lemma}

\begin{remark}
Here $D$ and $D^{i}$ are defined in Sect. \ref{sect:general_bounds} and $D^{i} = \{D \backslash z_i \cup z_i^{\prime} \}$.
\end{remark}


\textbf{Proof of Theorem \ref{thm:polybound}}.

\begin{proof}

First, for the given loss function $\ell$, we have that,
\begin{align}
\ell(\mathcal{A}, X_{ij}) & = M \bullet X_{ij} \nonumber \\
& \le \lambda_{\max} || x_i - x_j ||_2^2 \nonumber \\
& \le \frac{4KR\Gamma^2}{d}. \nonumber
\end{align}
Then due to the definition of uniform-RO stability, we have that
\begin{align}
\mathbb{E}_{\{D \bigcup z_i^{\prime}\}}\left[ |\ell(A_D,z_i) - \ell(A_{D^i},z_i)| \right] \le \beta. \nonumber
\end{align}

Therefore, according to the above lemma of variance bound, we can obtain that,
\begin{align}
& \mathbb{E}_{D}\left[ \left( \mathcal{R}(\mathcal{A}, D) - \mathcal{R}_{emp} (\mathcal{A}, D) \right)^2 \right] \nonumber \\
& \le \frac{8K^2R^2\Gamma^4}{nd^2} + \frac{12KR\Gamma^2\beta}{d}. \nonumber
\end{align}

Then based on Chebyshev's inequality, we can derive that,
\begin{align}
& Prob \left( \mathcal{R}(\mathcal{A}, D) - \mathcal{R}_{emp}(\mathcal{A}, D) \ge \epsilon \right) \nonumber \\
& \le \frac{\mathbb{E}_{D}\left[ \left( \mathcal{R}(\mathcal{A}, D) - \mathcal{R}_{emp} (\mathcal{A}, D) \right)^2 \right]}{\epsilon^2}. \nonumber \\
& \le \frac{1}{\epsilon^2}\left( \frac{8K^2R^2\Gamma^4}{nd^2} + \frac{12KR\Gamma^2\beta}{d} \right). \nonumber
\end{align}

Setting the right hand side of the above inequality as $\delta$, we thus have with probability at least $1 - \delta$ that,
\begin{align}
\mathcal{R}(\mathcal{A}, D) \le \mathcal{R}_{emp} (\mathcal{A}, D) + 2\Gamma\sqrt{\frac{KR}{d\delta} \left(\frac{2KR\Gamma^2}{nd} + 3\beta\right)}. \nonumber
\end{align}

By substituting the Uniform-Replace-One stability $\beta$, we can obtain the following specific bound,
\begin{align}
\mathcal{R}(\mathcal{A}, D) \le \mathcal{R}_{emp} (\mathcal{A}, D) + \frac{2R\Gamma^2}{d} \sqrt{\frac{2K}{\delta}\left( \frac{K}{n} + 6K + 6 \right)}. \nonumber
\end{align}

\end{proof}

\begin{remark}

As aforementioned, there is one counter-intuitive property of our generalization bound that it becomes tighter when the feature dimension $d$ increases. However, it is the case only if our previous full-rank assumption on $M$ holds. If this assumption is violated, \eg, in the sparse high dimensional feature space, the above result does not stand any more. Therefore, it rules out the possibility that trivially increasing the feature dimension by adding zeros will improve the generalization ability.

\end{remark}


\section{Impact of Parameters}

In this section, we study how the performance of our BDML algorithm varies with several important parameters, including distortion bound $K$, trace bound $R$ and width $\rho$. Except running time which needs large scale data, we experiment with all other parameters on the UCI Iris dataset. As in Sect. \ref{sect:classification}, we randomly split the dataset into 70\% for training and 30\% for testing and report the average test error and its standard deviation by repeating the random splits for $10$ times.

\subsection{Distortion Bound $K$}

We first study the effects of distortion bound $K$. We fix the number of iteration $T = 1000$, the margin of \textit{p}-BDML $\mu = 1$, the trace bound $R = 100$ and the width $\rho = 500$. And we report the mean condition number, mean test error and the standard deviation of test error. The results are listed in Table \ref{tbl:boundedDistortion}. From the table, we can find that with $K$ increases, the resultant mean condition number becomes larger. It is partly because that as $K$ increases, the bounded-distortion constraint becomes easier to satisfy, thus encouraging MWU method puts more weights on other constraints like the margin ones. Hence the learned metric embedding is more distorted to fitting the training data. Moreover, with $K$ increases, the test error first decreases and then increases which matches our analysis in Sect. \ref{sect:general_bounds}.

\begin{table}\label{tbl:boundedDistortion}
\centering
\begin{tabular}{c|c|c|c}
\hline
    $K$         & Mean Cond. Num.   & Mean Err.     & Std Err.($\pm$)   \\ \hline \hline
    1.0e+1      & 1.454             & 4.222         & 2.446             \\ \hline
    1.0e+2      & 2.165             & 3.333         & 2.160             \\ \hline
    1.0e+3      & 2.158             & 4.222         & 2.147             \\ \hline
    1.0e+4      & 37.567            & 5.333         & 2.859             \\ \hline
    1.0e+5      & 2962.094          & 6.222         & 3.443             \\
\hline
\end{tabular}
\caption{Performance varies with distortion bound $K$.}
\end{table}

\subsection{Trace Bound $R$ and Width $\rho$}

\addtolength{\tabcolsep}{-2pt}
\begin{table}\label{tbl:traceBound_width}
\centering
\footnotesize
\begin{tabular}{c|c|c|c|c}
\hline
    $R$         & $\rho$    & Mean Cond. Num.   & Mean Err.     & Std Err.($\pm$)   \\ \hline \hline
    1.0e+2      & 1.0e+2    & 2.682             & 4.222         & 2.210             \\ \hline
    1.0e+2      & 5.0e+2    & 1.638             & 4.222         & 1.640             \\ \hline
    1.0e+2      & 1.0e+3    & 2.354             & 5.111         & 1.500             \\ \hline
    1.0e+3      & 1.0e+3    & 2.768             & 5.778         & 2.608             \\ \hline
    1.0e+3      & 5.0e+3    & 1.611             & 4.889         & 1.753             \\ \hline
    1.0e+3      & 1.0e+4    & 2.360             & 7.111         & 2.295             \\
\hline
\end{tabular}
\caption{Performance varies with trace bound $R$ and width $\rho$.}
\end{table}

We now study the effects of trace bound $R$ and width $\rho$. These two parameters are correlated in a sense that the width $\rho$ should not be much smaller than the trace bound $R$ since otherwise the constraint of width, \ie, $\forall i$, $\left| {{J_i} \bullet {Y^{(t)}} - {h_i}} \right| \le \rho$ will not stand. We fix the number of iteration $T = 1000$, the margin of \textit{p}-BDML $\mu = 1$ and the distortion bound $K = 1000$. Same measurements are reported in Table \ref{tbl:traceBound_width}.

From this Table, we can see that if width $\rho$ is set to be much larger than the trace bound $R$, the resultant mean test error tends to become larger. This may due to the fact that with the same number of iteration, the larger the width, the smaller the overall quality of the solution of the MWU method deteriorates which is matched to the analysis in \cite{kale2007efficient}. On the other side, if width $\rho$ is nearly equal to the trace bound $R$, the aforementioned constraint of width will be violated sometimes. Therefore, in practice, we find that setting $R \approx 10d$ and $\rho \approx 5R$ yield good results. Here $d$ is the dimension of the input feature.

\subsection{Running Time}

\begin{figure}[t]
\centering
\includegraphics[width=0.70\linewidth]{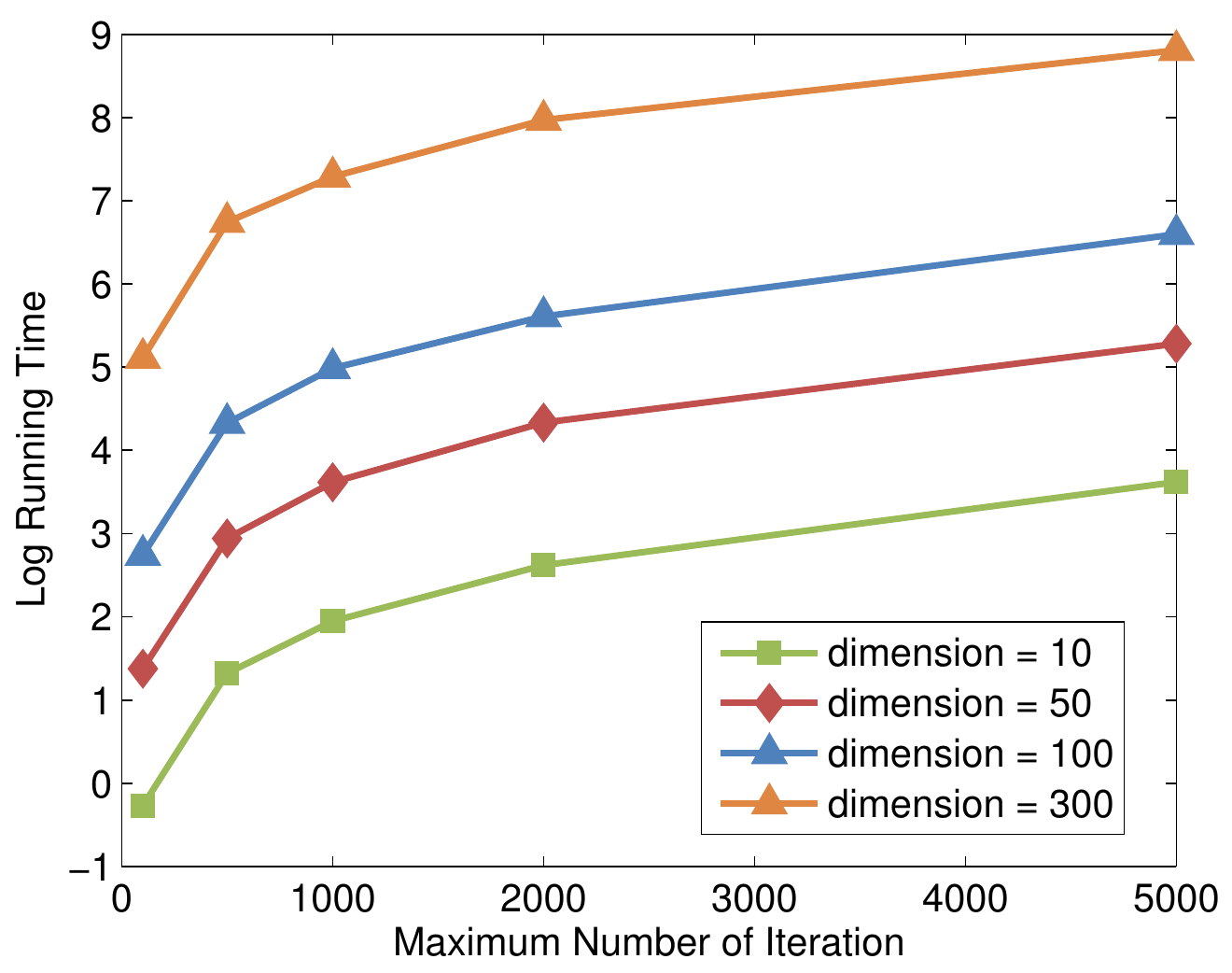}
\caption{Runing Time of MWU solver.}
\label{fig:runTime}
\end{figure}

We now investigate the running time on the LFW dataset due to its high-dimensional feature. Since the main component of our BDML algorithm is the MWU method, we thus study how its running time varies with respect to the maximum number of iteration and dimension of input feature. The trace bound $R$ and width $\rho$ are both fixed as $3d+1$ as aforementioned. We implement the algorithm as a single-thread \textsc{Matlab} program. And all our experiments are conducted on a server with Intel Xeon E5 CPU(2.6GHz) and 128G RAM. In particular, we test following values of dimension $d$ of input feature, 10, 50, 100 and 300. And for each dimension, we set the maximum number of iteration as 100, 500, 1000, 2000 and 5000 and keep track of the corresponding running time. The natural logarithmic of all results are plotted in Fig.~\ref{fig:runTime}. Generally, for 100-dim feature, it takes around 146s to finish 1000 iterations of MWU.

\begin{table*}[t]
\begin{center}
\begin{tabular}{c|c|c|c|c|c|c}
\hline
  Methods           & OrigFeat & SGF & GFK(PCA) & LMNN  & DML-eig   & \textit{t}-BDML      \\ \hline \hline
  A $\rightarrow$ C & 22.6(0.3)	& 35.3(0.5)	& 35.6(0.4)	& 35.7(0.5)	& 35.0(0.7)	& \textbf{37.2}(0.3) \\ \hline
  A $\rightarrow$ D & 22.2(0.4)	& 30.7(0.8)	& \textbf{35.2}(0.9)	& 32.3(1.0)	& 27.6(1.1)	& 33.0(1.0) \\ \hline
  A $\rightarrow$ W & 23.5(0.6)	& 31.0(0.7)	& 34.4(0.9)	& 32.9(0.8)	& 28.9(0.7)	& \textbf{35.2}(1.0) \\ \hline
  C $\rightarrow$ A & 20.8(0.4)	& 36.8(0.5)	& \textbf{36.9}(0.4)	& 33.8(0.7)	& 33.7(0.7)	& 35.2(0.7) \\ \hline
  C $\rightarrow$ D & 22.0(0.6)	& 32.6(0.8)	& \textbf{35.2}(1.0)	& 31.5(1.6)	& 32.7(1.3)	& 33.4(1.3) \\ \hline
  C $\rightarrow$ W & 19.4(0.7)	& 30.6(0.8)	& \textbf{33.7}(1.1)	& 26.0(1.2)	& 29.2(1.3)	& 29.6(0.8) \\ \hline
  D $\rightarrow$ A & 27.7(0.4)	& 32.0(0.4)	& 32.5(0.5)	& 33.7(0.4)	& 33.4(0.3)	& \textbf{37.1}(0.6) \\ \hline
  D $\rightarrow$ C & 24.8(0.4)	& 29.4(0.5)	& 29.8(0.3)	& 29.4(0.5)	& 29.8(0.3)	& \textbf{32.6}(0.6) \\ \hline
  D $\rightarrow$ W & 53.1(0.6)	& 66.0(0.5)	& 74.9(0.6)	& 75.1(0.8)	& 78.2(0.8)	& \textbf{78.6}(0.7) \\ \hline
  W $\rightarrow$ A & 20.7(0.6)	& 27.5(0.5)	& 31.1(0.8)	& 30.8(0.7)	& 32.5(0.9)	& \textbf{33.2}(0.8) \\ \hline
  W $\rightarrow$ C & 16.1(0.4)	& 21.7(0.4)	& 27.2(0.5)	& 26.3(0.7)	& 27.0(0.5)	& \textbf{28.5}(0.6) \\ \hline
  W $\rightarrow$ D & 37.3(1.2)	& 54.3(1.2)	& 70.6(0.9)	& 67.6(1.0)	& 72.4(0.6)	& \textbf{73.8}(0.6) \\
\hline
\end{tabular}
\end{center}
\caption{Comparison of unsupervised domain adaptation. Mean accuracy (\%) and standard error (inside parentheses) are reported. The best performance is denoted in bold type.}
\label{tbl:DomainAdaptationUnsupervisedFull}
\end{table*}

\section{Full Results of Experiments}

We in this section demonstrate the comprehensive results of our experiments.

\subsection{Domain Adaptation}

\begin{table*}[t]
\begin{center}
\begin{tabular}{c|c|c|c|c|c|c|c|c}
\hline
  Methods & OrigFeat & ITML & SGF & GFK(PCA) & LMNN & DML-eig & MMDT  & \textit{t}-BDML      \\ \hline \hline
  A $\rightarrow$ C & 24.0(0.3) & 27.3(0.7)	& 37.7(0.5)	& 37.8(0.4)             & 36.6(0.6)	& 27.8(0.7)             & 36.4(0.8)                 & \textbf{38.8}(0.3) \\ \hline
  A $\rightarrow$ D & 28.1(0.6) & 33.7(0.9)	& 34.5(1.1)	& 47.0(1.2)             & 43.8(1.2) & 33.0(0.9)             & \textbf{56.7}(1.3)        & 46.5(0.9) \\ \hline
  A $\rightarrow$ W & 31.6(0.6) & 36.0(1.0)	& 37.9(0.7)	& 53.7(0.8)             & 49.6(0.9)	& 40.5(1.0)	            & \textbf{64.6}(1.2)        & 55.8(1.1) \\ \hline
  C $\rightarrow$ A & 23.1(0.4) & 33.7(0.8)	& 40.2(0.7)	& 42.0(0.5)             & 43.3(0.5)	& 43.3(0.6)	            & \textbf{49.4}(0.8)        & 44.6(0.6) \\ \hline
  C $\rightarrow$ D & 26.5(0.7) & 35.0(1.1)	& 36.6(0.8)	& 49.5(0.9)             & 50.3(1.3)	& 45.1(1.6)	            & \textbf{56.5}(0.9)        & 54.0(1.1) \\ \hline
  C $\rightarrow$ W & 25.2(0.8) & 34.7(1.0)	& 37.2(0.9)	& 54.2(0.9)             & 56.2(1.5)	& 58.9(1.2)	            & \textbf{63.8}(1.1)        & 56.0(1.0) \\ \hline
  D $\rightarrow$ A & 31.3(0.7) & 30.3(0.8)	& 39.2(0.7)	& 45.0(0.7)             & 42.0(0.7)	& 43.4(0.6)             & \textbf{46.9}(1.0)        & 43.9(0.6) \\ \hline
  D $\rightarrow$ C & 22.4(0.5) & 22.5(0.6)	& 30.2(0.7)	& 32.7(0.4)             & 33.4(0.4)	& 31.9(0.4)	            & 34.1(0.8)                 & \textbf{35.4}(0.3) \\ \hline
  D $\rightarrow$ W & 55.5(0.7) & 55.6(0.7)	& 69.5(0.9)	& 78.7(0.5)             & 78.6(0.7)	& 80.5(0.9)	            & 74.1(0.8)                 & \textbf{83.8}(0.5) \\ \hline
  W $\rightarrow$ A & 30.8(0.6) & 32.3(0.8)	& 38.2(0.6)	& 42.8(0.7)             & 42.3(0.6)	& 40.8(0.7)	            & \textbf{47.7}(0.9)        & 44.8(0.6) \\ \hline
  W $\rightarrow$ C & 20.8(0.5) & 21.7(0.5)	& 29.2(0.7)	& 32.8(0.7)             & 32.2(0.7)	& 32.8(0.6)             & 32.2(0.8)                 & \textbf{33.3}(0.6) \\ \hline
  W $\rightarrow$ D & 44.3(1.0) & 51.3(0.9)	& 60.6(1.0)	& 75.0(0.7)             & 72.8(1.1)	& 76.8(0.9)	            & 67.0(1.1)                 & \textbf{79.2}(0.7) \\
\hline
\end{tabular}
\end{center}
\caption{Comparison of semi-supervised domain adaptation. Mean accuracy (\%) and standard error (inside parentheses) are reported. The best performance is denoted in bold type.}
\label{tbl:DomainAdaptationSemisupervisedFull}
\end{table*}

For domain adaptation, we show the full results of all 12 possible combinations of source and target domains. In particular, unsupervised and semi-supervised experiments are listed in Table~\ref{tbl:DomainAdaptationUnsupervisedFull} and Table~\ref{tbl:DomainAdaptationSemisupervisedFull} respectively. We include the state-of-art results of max-margin domain transformations (MMDT)~\cite{hoffman2013efficient} on this dataset in Table~\ref{tbl:DomainAdaptationSemisupervisedFull}. Note that the comparison with MMDT is somewhat unfair for our method, because it exploits the discrimination power of a max-margin classifier, whereas ours is the simple distance metric learning based 1-NN classifier. However, it is promising that, with such simple classifier, our BDML still achieves state-of-the-art results in some subtasks.

\subsection{Face Verification}

\begin{figure}[t]
\centering
\includegraphics[width=0.70\linewidth]{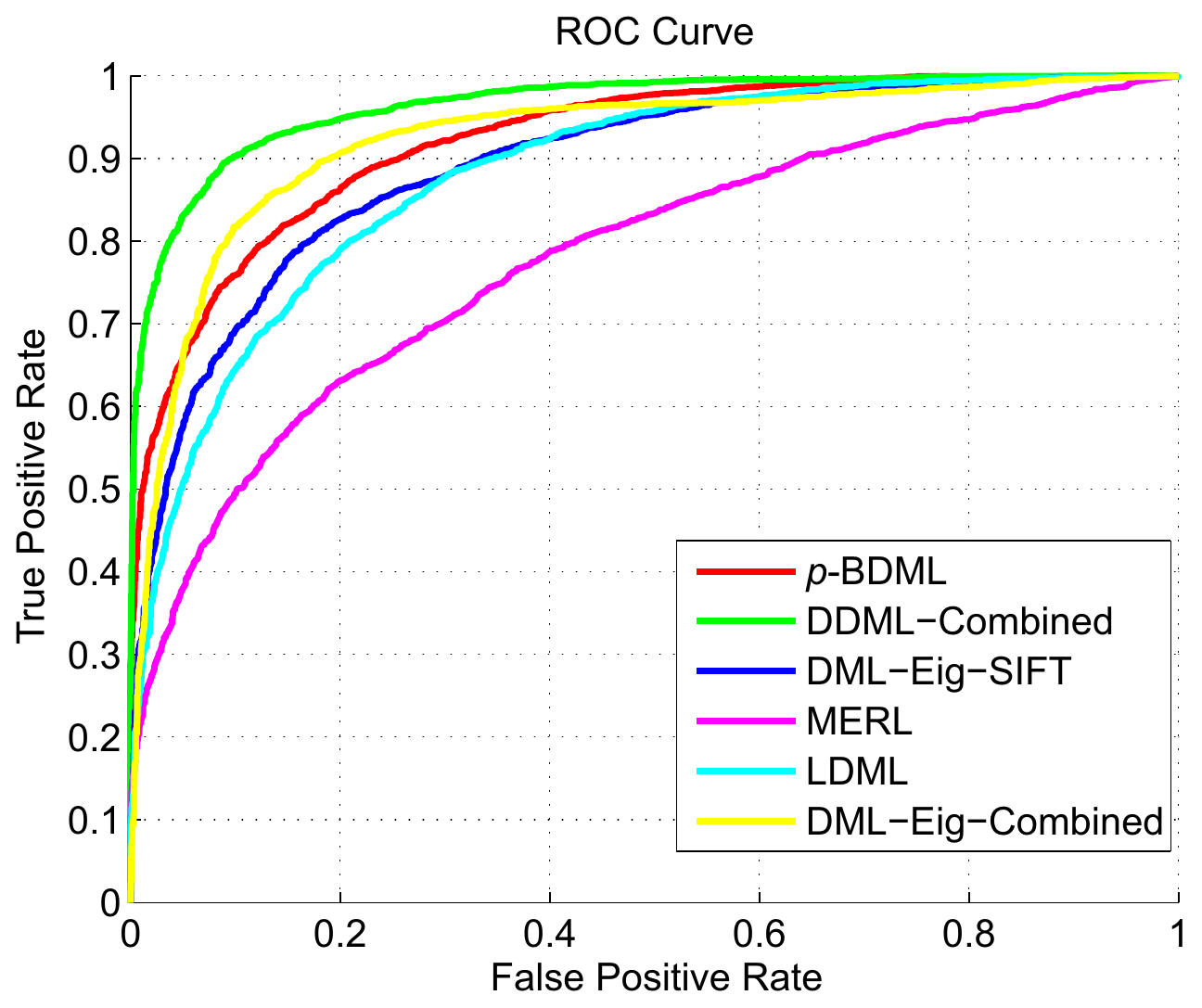}
\caption{ROC curves on LFW dataset.}
\label{fig:ROC}
\end{figure}

In this section, we present full experimental comparisons on LFW dataset. In Table~\ref{tbl:LFW_Full}, we list various published results on ``Image-Restricted, Label-Free Outside Data" setting of LFW dataset. Specifically, the abbreviations of these algorithms are MERL~\cite{huang2008lfw}, Xing~\cite{xing2003distance}, ITML~\cite{davis2007information}, LDML~\cite{guillaumin2009you}, DML-eig-SIFT~\cite{ying2012distance}, Sub-ITML~\cite{cao2013similarity}, KISSME~\cite{kostinger2012large}, Sub-ML~\cite{cao2013similarity}, LBP+CSML~\cite{nguyen2011cosine}, DML-eig-Combined~\cite{ying2012distance}, Convolutional DBN~\cite{huang2012learning}, Sub-SML~\cite{cao2013similarity} and DDML-Combined~\cite{hu2014DDML}. Among them, ``Convolutional DBN" and ``DDML-Combined" are deep learning based methods. Suffix ``Combined" means the method uses multiple descriptors, e.g., SIFT~\cite{lowe2004distinctive}, LBP~\cite{ojala2002multiresolution}, TPLBP~\cite{wolf2008descriptor}, etc. From this table, we can find that, although using only dimension-reduced SIFT feature, our BDML algorithm achieves comparable results with other feature-combined and non-metric learning based ones. Moreover, our method achieves the least stand errors compared to others which indicates that our BDML produces stable metrics. The ROC curves versus others are plotted in Fig.~\ref{fig:ROC}.

\begin{table}[t]
\begin{center}
\begin{tabular}{c|c|c}
\hline
  Methods               & Mean Acc.             & Std Err.($\pm$)   \\ \hline \hline
  MERL                  & 0.7052                & 0.0060            \\ \hline
  Xing                  & 0.7593                & 0.0059            \\ \hline
  ITML                  & 0.7812                & 0.0045            \\ \hline
  LDML                  & 0.7927                & 0.0060            \\ \hline
  DML-eig-SIFT          & 0.8127                & 0.0230            \\ \hline
  Sub-ITML              & 0.8145                & 0.0046            \\ \hline
  KISSME                & 0.8308                & 0.0056            \\ \hline
  Sub-ML                & 0.8330                & 0.0026            \\ \hline
  LBP+CSML              & 0.8557                & 0.0052            \\ \hline
  DML-eig-Combined      & 0.8565                & 0.0056            \\ \hline
  \emph{p}-BDML         & 0.8632                & 0.0022            \\ \hline
  Convolutional DBN     & 0.8777                & 0.0062            \\ \hline
  Sub-SML               & 0.8973                & 0.0038            \\ \hline
  DDML-Combined         & \textbf{0.9068}       & 0.0141            \\ \hline
\end{tabular}
\end{center}
\caption{Full comparison on ``Image-Restricted, Label-Free Outside Data" setting of LFW dataset.}
\label{tbl:LFW_Full}\vspace{-.2cm}
\end{table}

\section{Useful Tail Bound for Chi-square Variables}

We list the following sharp tail bound for chi-square variables.



\begin{lemma}\label{chi_square_tail_bound_3}
\cite{laurent2000adaptive} Let $X \sim \chi_d^2$ and $\epsilon \ge 0$, then
\begin{align}
P(X - d \ge 2\sqrt{d\epsilon} + 2\epsilon) \le \exp{\left(-\epsilon\right)} \nonumber \\
P(X - d \le -2\sqrt{d\epsilon}) \le \exp{\left(-\epsilon\right)}. \nonumber
\end{align}
\end{lemma}


\bibliography{BDML}
\bibliographystyle{ieee}

\end{document}